%% file: main.tex
\documentclass{article}

\PassOptionsToPackage{numbers,sort&compress}{natbib}
\usepackage[preprint]{neurips_2026}

\def \name {FATE}

\usepackage[utf8]{inputenc} % allow utf-8 input
\usepackage[T1]{fontenc}    % use 8-bit T1 fonts
\usepackage{hyperref}       % hyperlinks
\usepackage{url}            % simple URL typesetting
\usepackage{booktabs}       % professional-quality tables
\usepackage{amsfonts}       % blackboard math symbols
\usepackage{nicefrac}       % compact symbols for 1/2, etc.
\usepackage{microtype}      % microtypography
\usepackage{xcolor}         % colors
\usepackage{wrapfig}        % wraps around the texts

\usepackage{natbib}
\usepackage{amsmath}
\usepackage{amssymb}
\usepackage{verbatim}
\usepackage{booktabs}
\usepackage{multirow}
\usepackage{graphicx}
\usepackage{soul}
\usepackage{algorithm}
\usepackage{algpseudocode}
\usepackage{titletoc}
\usepackage{subcaption}
\usepackage{pifont}
\newcommand{\cmark}{\ding{51}} % checkmark
\newcommand{\xmark}{\ding{55}} % cross mark

\algrenewcommand\algorithmicrequire{\textbf{Input:}}
\algrenewcommand\algorithmicensure{\textbf{Output:}}

\usepackage{amsthm}
\theoremstyle{plain}

\newtheorem{theorem}{Theorem}

\title{FATE: Pillar Encoding and Frequency-Aware Training for Event-Based Object Detection}

\author{
  Md Tawheedul Islam Bhuian \quad Kyoung-Don Kang \\
  \smallskip
  \small School of Computing \\
  \small State University of New York at Binghamton \\
  \small \texttt{\{mislambhuian, kang\}@binghamton.edu}
}

\begin{document}

\maketitle

% ---- Main Paper ----

\begin{abstract}
  \input{sections/0.abstract}
\end{abstract}

\section{Introduction}
\label{sec:intro}
\input{sections/1.intro3} 

\section{Related Work}
\label{sec:related}
\input{sections/2.related3}

\section{Method}
\label{sec:method}
\input{sections/3.method2}

\section{Evaluation}
\label{sec:eval}
\input{sections/4.eval5}

%\subsection{Results}
%\input{sections/5.results}

\subsection{Ablation Study}
\label{sec:ablation}
\input{sections/6.ablation2}

\section{Conclusion}
\label{sec:conc}
\input{sections/7.conclusion}

% ---- Bibliography ----
\newpage
\bibliographystyle{plainnat}
\bibliography{main}

% ---- Appendix or Supplementary Materials ----

\newpage
\appendix
\setcounter{equation}{0}
\renewcommand{\theequation}{A.\arabic{equation}}
\input{sections/8.appendix}

% ---- Checklist ----

\newpage
%\section{Checklist}
% \input{sections/9.checklist}

\end{document}

%% file: sections/0.abstract.tex
Event cameras are bio-inspired sensors that asynchronously capture logarithmic intensity changes, offering inherent advantages in high-speed and high-dynamic-range scenarios. However, the sparse and asynchronous nature of event streams poses a fundamental challenge for modern deep learning architectures. To enable compatibility with standard models, most existing approaches partition the accumulation window into fixed temporal sub-bins. While effective for spatial processing, this internal discretization discards fine-grained temporal structure and constrains inference to the low temporal frequencies imposed by training supervision. To address this limitation, we propose FATE, a unified framework built upon a novel Pillar Encoding (PE). While operating over discrete macro-accumulation windows dictated by the target frequency, PE avoids internal temporal sub-binning. It organizes events into spatial pillars and approximates their intra-window evolution via projection onto a continuous-time orthogonal polynomial basis. This formulation yields an $L^2$-optimal representation that retains rich temporal dynamics in a dense pseudo-image, mitigating information loss under sparse event conditions.
To fully leverage this representation, we introduce Frequency-Aware Training (FAT), a soft mean-teacher curriculum that generates temporally dense pseudo-labels, effectively bridging the mismatch between low-frequency supervision and high-frequency inference. Extensive experiments demonstrate that FATE generalizes across architectural paradigms and consistently outperforms strong baselines. It enables robust object detection at high temporal resolutions up to 200 Hz, while incurring minimal overhead in parameter count and inference latency.

%Event cameras are bio-inspired sensors that asynchronously capture logarithmic intensity changes, offering inherent advantages in high-speed and high-dynamic-range scenarios. However, the sparse and asynchronous nature of event streams poses a fundamental challenge for modern deep learning architectures. To enable compatibility with standard models, most existing approaches aggregate events into fixed temporal bins or surfaces. While effective for spatial processing, such discretization discards fine-grained temporal structure and constrains inference to the low temporal frequencies implicitly imposed by training supervision.To address this limitation, we propose FATE, a unified framework built upon a novel Pillar Encoding (PE). Instead of discretizing time into coarse bins, PE organizes events into spatial pillars and approximates their temporal evolution via projections onto a continuous-time orthogonal polynomial basis within each local window. This formulation retains rich temporal dynamics in a dense pseudo-image representation, mitigating loss under sparse event conditions. To better utilize this representation, we propose Frequency-Aware Training (FAT), a soft mean-teacher curriculum that generates temporally dense pseudo-labels, bridging the mismatch between low-frequency supervision and high-frequency inference. Extensive experiments demonstrate that FATE consistently outperforms strong baselines, enabling robust object detection at temporal resolutions up to 200 Hz while incurring minimal overhead in parameter count and inference latency.

%% file: sections/1.intro3.tex
Event cameras introduce a fundamentally different paradigm for visual perception. Unlike conventional frame-based cameras that record absolute intensity at fixed intervals, event cameras operate asynchronously, emitting events only in response to local logarithmic brightness changes. This sensing mechanism provides fine-grained temporal resolution, high dynamic range, and minimal motion blur~\cite{gallego2020eventSurvey}. These properties make event cameras particularly well-suited for applications such as autonomous driving, robotics, and high-speed tracking, where conventional sensors often degrade under rapid motion or challenging lighting conditions.

Despite these advantages, effectively leveraging event data for downstream tasks such as object detection remains challenging. Modern architectures, including Convolutional Neural Networks (CNNs) and Vision Transformers (ViTs), assume dense, grid-structured inputs, rendering raw asynchronous event streams incompatible. Consequently, most approaches partition the macro-accumulation window into fixed temporal sub-bins or voxel grids to form dense representations~\cite{gehrig2019end, rebecq2019events}. While this enables the use of standard backbones, it introduces a key limitation: internal temporal discretization. Aggregating events via rigid sub-binning discards fine-grained temporal structure and implicitly constrains inference to low temporal frequencies present during training (e.g., 20 Hz), limiting performance at higher temporal resolutions.

To address this limitation, we propose \textbf{FATE}, a unified framework for high-frequency event-based object detection. At its core is \textbf{Pillar Encoding (PE)}. While PE operates over discrete macro-accumulation windows dictated by the target frequency, it avoids internal temporal sub-binning. By discretizing the spatial domain into pillars, PE models the within-pillar event evolution as a continuous-time signal. Specifically, asynchronous event features within each pillar are approximated via projection onto a truncated set of orthogonal Legendre polynomials, yielding a compact set of coefficients that encode temporal dynamics and mitigate information loss under sparse event conditions. This construction corresponds to an orthogonal projection onto a finite-dimensional subspace, yielding the minimizer of the $L^2$ reconstruction error under a fixed coefficient budget. As a result, PE preserves rich temporal structure while producing a dense pseudo-image compatible with standard architectures.

%unique 

While PE enables continuous-time modeling, learning such representations introduces a second challenge: the scarcity of high-frequency annotations. Event-based datasets typically provide labels at low frame rates, creating a mismatch between training supervision and high-frequency inference. To bridge this gap, we propose \textbf{Frequency-Aware Training (FAT)}. FAT first generates temporally dense pseudo-labels via tracking-by-detection, and then applies a soft mean-teacher curriculum that progressively exposes the student model to higher-frequency inputs while enforcing consistency with low-frequency supervision.

Together, PE and FAT form a complementary framework that addresses both representation and supervision mismatches. We evaluate FATE on standard event-based object detection benchmarks and observe consistent improvements over strong baselines, with more pronounced gains in high-frequency regimes.

\paragraph{Contributions.}
\begin{itemize}
\item \textbf{Pillar Encoding:} We introduce an event representation that avoids internal temporal sub-binning by modeling intra-window dynamics as continuous-time functions via orthogonal polynomial projections. Our formulation yields an $L^2$-optimal approximation under truncation, efficiently preserving fine-grained temporal structure within the accumulation window.

\item \textbf{Frequency-Aware Training:} We propose a training strategy that leverages temporally dense pseudo-labels and a mean-teacher curriculum to mitigate train--test frequency mismatch, enabling effective high-frequency inference without additional manual annotation.

\item \textbf{Robust Empirical Performance:} FATE consistently outperforms strong baselines on major benchmarks, enabling robust object detection at temporal resolutions up to 200~Hz with minimal computational overhead, utilizing less than 0.15M additional parameters and a $\sim$1\% end-to-end inference latency increase.
\end{itemize}

%% file: sections/2.related3.tex
%\subsection{Event Representations}

\textbf{Event Representations.}
Event cameras output an asynchronous stream of events rather than synchronous intensity frames, yielding a signal that is spatially sparse and temporally irregular. Early approaches transformed these events into handcrafted, frame-like tensors for compatibility with standard vision architectures. Event histograms~\cite{sironiHATSHistogramsAveraged2018} and time surfaces~\cite{lagorceHOTSHierarchyEventBased2017, benosmanEventBasedVisualFlow2013} aggregate events per pixel, discarding intermediate events prior to the most recent one. Stacked histograms~\cite{perotLearningDetectObjects} and voxel grids~\cite{zhuUnsupervisedEventBasedLearning2019} partition timestamps into discrete temporal bins, while TORE volumes~\cite{baldwinTimeOrderedRecentEvent2023} retain the $k$ most recent timestamps. Fundamentally, these representations rely on windowing or binning hyperparameters that trade temporal precision for stable tensor dimensions.

Alternatively, point-based methods preserve temporal fidelity by processing events as irregular point sets. EventNet~\cite{sekikawaEventNetAsynchronousRecursive2019} recursively updates global features, while graph-based methods~\cite{biGraphBasedSpatioTemporalFeature2020, dengVoxelGraphCNN2022} apply message passing over spatiotemporal graphs. The asynchronous-to-synchronous (A2S) paradigm seeks a middle ground, encoding events asynchronously into a state that can be sampled by a synchronous backbone~\cite{kamalAssociativeMemoryAugmented2023, turrero2024alerttransformer}. However, recent findings~\cite{haoMaximizingAsynchronicityEventbased2025} suggest that current A2S encoders still produce weaker representations relative to dense methods.

Learned grid representations, such as Matrix-LSTM~\cite{canniciDifferentiableRecurrentSurface2020} and differentiable event-to-grid frameworks~\cite{gehrig2019end}, adapt the encoding to downstream tasks. However, these methods still operate on dense pixel grids and rely on pooling or recurrence, biasing representations toward high-density regions and limiting the preservation of continuous temporal dynamics.

Our work belongs to this family of learned representations but departs from dense temporal discretization. 
We draw inspiration from PointPillars~\cite{langPointPillarsFastEncoders2019, liTinyPillarNet2024}, which collapses irregular point clouds into spatial columns. While EventPillars~\cite{fanEventPillarsPillarbasedEfficient} applies this idea to events using handcrafted statistics and primarily focuses on classification settings, we instead model temporal evolution for object detection in a learned manner. 
Building on orthogonal polynomial projections~\cite{smith2023simplified}, we approximate event dynamics within each pillar as a continuous trajectory, producing a compact and expressive representation robust to sparse event distributions.

%\subsection{Event-Based Architectures}
\textbf{Event-Based Architectures.}
Driven by these representations, various architectures have been proposed for event-based object detection. Early CNN-based detectors~\cite{jiangMixedFrameEventDriven2019, perotLearningDetectObjects} applied standard networks to frame-like inputs, discarding the asynchronous nature of events. Subsequent works introduced explicit temporal modeling: RED~\cite{perotLearningDetectObjects} employs convolutional LSTMs, ASTMNet~\cite{liAsynchronousSpatioTemporalMemory2022} leverages asynchronous spatio-temporal memory, and RVT~\cite{gehrigRecurrentVisionTransformers2023} combines multi-scale transformers with recurrent processing for efficient detection at a fixed frequency (20 Hz). More recently, PLEIADES~\cite{pei2026pleiades} proposes continuous-time temporal kernels parameterized via orthogonal polynomial bases.

Recent advances improve efficiency and temporal reasoning. These include group-wise event tokenization~\cite{pengGETGroupEvent2023}, scene-adaptive sparse attention~\cite{pengSceneAdaptiveSparse2024}, and state-space models~\cite{zubicStateSpaceModels2024, yangSmamba2025} for scalable long-range temporal modeling. Recent works such as EvRT-DETR~\cite{torbunovEvRTDETRLatentSpace} adapt pretrained image-based models to event data, while others explore spiking or hybrid ANN–SNN architectures for energy-efficient processing~\cite{suDeepDirectlyTrainedSpiking2023, aydinHybridANNSNNArchitecture2024, liHDI2024}. Additionally, multimodal transformer frameworks such as SODFormer~\cite{liSODFormerStreamingObject2023} have emerged to fuse asynchronous event streams with complementary frame-based information.

Despite these advances, most architectures are trained and evaluated at a fixed, relatively low temporal frequency (e.g., 20\,Hz), leading to degraded performance at higher frequencies with short accumulation windows~\cite{gehrigLowlatencyAutomotiveVision2024}. To address related issues, FlexEvent~\cite{luFlexEventFlexibleEventFrame2025} enables frequency adaptation via event-frame fusion but relies on synchronized RGB inputs. LEOD~\cite{wuLEODLabelEfficientObject2024} improves supervision through self-training but does not explicitly address distribution shifts across multiple frequencies.

More recently, backbones focus on sparsity and high-frequency robustness. SMamba~\cite{yangSmamba2025} leverages state-space modeling, SSLA-Det~\cite{hao2026lowlatency} introduces asynchronous linear attention for sparse processing, EMF~\cite{khan2025emf} models per-pixel temporal dynamics via event progression, SEED~\cite{wang2025sparse} targets neuromorphic efficiency, and MvHeat-DET~\cite{wang2025object} employs heat-conduction mechanism to balance accuracy and efficiency. However, these methods primarily focus on backbone design and do not explicitly address pillar encoding or multi-frequency supervision. In contrast, FATE explicitly targets this gap by providing pillar encoding and supervision across multiple temporal resolutions without requiring paired RGB data or manual annotations of high-frequency events.

%It dynamically adapts learning to varying event rates without requiring paired RGB data or high-frequency manual annotations.

%% file: sections/3.method2.tex
\subsection{Pillar Encoding}
\label{sec:pillar_encoding}
%We first transform a raw event stream into an image-like representation. Prior event representations, such as stacked 2D histograms~\cite{gallego2020eventSurvey}, discretize events into rigid temporal bins. In short windows, event streams become sparse, and binning further attenuates the already limited signals. To address the issue, the proposed PE discretizes the spatial domain into pillars and models the events within each pillar using a polynomial approximation of continuous-time dynamics. The steps of PE are illustrated in Figure~\ref{fig:pe} and discussed in sequence, as follows.

We transform raw event streams into an image-like representation. Prior methods (e.g., stacked 2D histograms~\cite{gallego2020eventSurvey}) rely on rigid temporal binning, which exacerbates sparsity in short windows. To address this, our PE discretizes space into pillars and models events within each pillar via a continuous-time polynomial approximation. The PE pipeline is illustrated in Figure~\ref{fig:pe} and described next.

%\paragraph{Event Pillarization.} 

\textbf{Event Pillarization.}
Given a target inference frequency $f$, we define a macro-accumulation window $[t_1, t_2]$ of duration $\Delta t = 1/f$. The event stream within this window is denoted as a set of spatiotemporal points $\mathcal{E}$. To process these points, we first discretize the spatial domain into an evenly spaced grid in the $x-y$ plane, partitioning $\mathcal{E}$ into a set of discrete spatial pillars $\mathcal{P}$ indexed by $j \in \{1,\ldots,P\}$. Crucially, we retain the continuous temporal precision of each event within its respective pillar. To ensure numerical stability and align with the canonical domain of the Legendre polynomials, the absolute timestamp $t_i \in [t_1, t_2]$ of each event $e_i \in \mathcal{E}$ is normalized and centered as $\tau_i = 2 \left( \frac{t_i - t_1}{\Delta t} \right) - 1$, ensuring $\tau_i \in [-1,1]$ and the center = 0.

\begin{figure}[htb]
    \vspace{-0.35cm}
    \centering
    \includegraphics[width=0.8\linewidth]{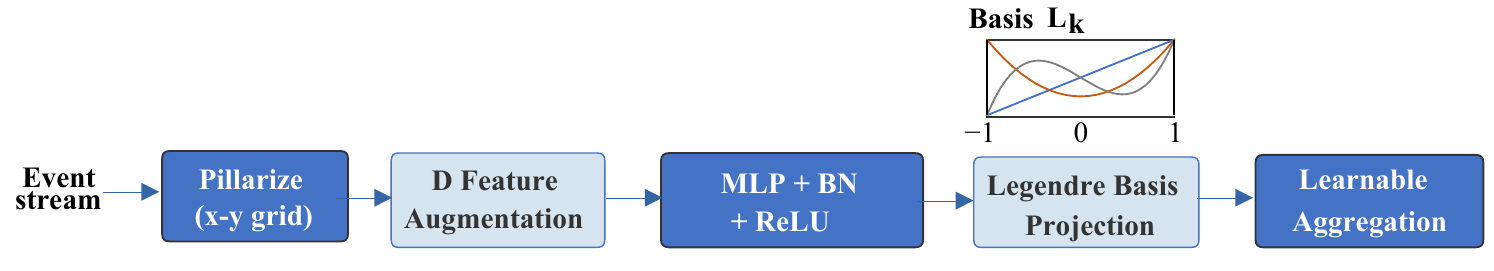}
    \caption{Pillar Encoding of FATE}
    \label{fig:pe}
    \vspace{-0.35cm}
\end{figure}

Each event is augmented with local spatial and temporal offsets to provide a stronger inductive bias. Offsets from the pillar-wise event mean capture locally occupied centroids and temporal skew, while offsets from the pillar center define a consistent local coordinate. Together, these features encode both the spatial distribution and temporal layout of events within each pillar.

For the $n$-th event in pillar $j$, the resulting augmented representation is defined as a $D$ dimensional feature vector:
\begin{equation}
l_{j,n} = (x_{j,n}, y_{j,n}, \tau_{j,n}, p_{j,n}, x_{j,n,m}, y_{j,n,m}, \tau_{j,n,m}, x_{j,n,c}, y_{j,n,c}),
\end{equation}
where $(x_{j,n},y_{j,n})$ denotes the location,
$\tau_{j,n}\in[-1,1]$ denotes the normalized timestamp of the $n$-th event in pillar $j$ sorted in ascending timestamp order,
and $p_{j,n}$ indicates the polarity ($+1$ or $-1$). 
The subscript $m$ denotes the offset from the arithmetic mean of all events within the pillar.  
Moreover, let $(x_{j, c}, y_{j, c})$ denote the predefined center of the $x$-axis and $y$-axis of pillar $j$. Given that, $x_{j,n,c} = x_{j,n} - x_{j, c}$, $y_{j,n,c} = y_{j,n} - y_{j, c}$.

%while the subscript $c = (x_{j, c}, y_{j, c}, \tau_{j, c}$ denotes the offset from the predefined spatial and temporal center of the pillar.

Due to the sparsity of event streams, the set of pillars is predominantly empty, and a non-empty pillar typically contains only a small number of events. We exploit this sparsity by capping the number of non-empty pillars per sample to $P$ and the number of events per pillar to $N$, yielding a dense tensor representation of size $(D, P, N)$. If pillar $j$ contains fewer than $N$ events, it is zero-padded by appending dummy events to the end of its events already sorted in timestamp-order and assigning a mask value $m_{j,n} = 0$ to dummy event $n$ in pillar $j$ to avoid any processing.

Conversely, if a spatial pillar contains more than $N$ events---a rare case under the high-frequency sparse regime we target---we apply uniform random subsampling, following~\cite{langPointPillarsFastEncoders2019}. We favor random subsampling over deterministic strategies to avoid temporal bias: truncation (e.g., retaining only the most recent events) would discard early temporal dynamics and distort the underlying event distribution. In contrast, random subsampling preserves the statistical properties of the event stream, enabling an effective approximation of the intra-window trajectory under our polynomial basis.

%\paragraph{Event-Wise Feature Embedding.} 
\textbf{Event-Wise Feature Embedding.}
Following pillarization, we model each pillar as a \textit{continuous-time signal}, rather than an unordered set of events or a stack of fixed temporal bins. This choice is motivated by the fact that different objects are characterized not only by the presence of events but also by the trajectory of how their features evolve over time within the pillar. 

First, we assemble the augmented event vectors into a dense input tensor $\mathbf{X} \in \mathbb{R}^{D \times P \times N}$. For each event, a shared event-wise MLP, followed by BatchNorm and ReLU, maps the $D$-dimensional vector into a higher-dimensional latent space:
\begin{equation}
\mathbf{H} = \sigma\!\left(\mathrm{BN}\!\left(\mathbf{W}\mathbf{X}+\mathbf{b}\right)\right),\quad
\mathbf{H}\in\mathbb{R}^{C\times P\times N}
\end{equation}
where $C$ is the latent feature dimension. 
Let $H_{c,j,n}$ denote the feature value of channel $c$ for the $n$-th event in pillar $j$. 
%Let $\tau_{j,n}\in[-1,1]$ denote the normalized timestamp of the $n$-th event in pillar $j$, sorted in ascending order. 
We view each discrete feature $H_{c,j,n}$ as a point-sample of an underlying continuous function $h_{c,j}(\tau)$ that evolves over time within the pillar, such that $H_{c,j,n} \approx h_{c,j}(\tau_{j,n})$. 

%\paragraph{Continuous-Time Temporal Encoding via Polynomial Approximation.} 
\textbf{Continuous-Time Temporal Encoding via Polynomial Approximation.}
To rigorously model these within-pillar dynamics, we project the latent feature trajectory $h_{c,j}(\tau)$ onto a set of orthogonal polynomial basis functions defined over the Hilbert space $\mathcal{H} = L^2([-1,1])$. We select the Legendre polynomials \cite{abramowitz1964handbook}, denoted as $L_k(\tau)$ for degree $k \in \{0, \dots, K-1\}$, because they form a complete orthogonal basis on $[-1,1]$ satisfying:
\begin{equation}
\langle L_k, L_m \rangle = \int_{-1}^{1} L_k(\tau) L_m(\tau) d\tau = \frac{2}{2k+1} \delta_{km},
\end{equation}
where $\delta_{km}$ is the Kronecker delta. By the Hilbert projection theorem \cite{kreyszig1978introductory, luenberger1997optimization}, any square-integrable temporal signal can be exactly expanded as an infinite series $h_{c,j}(\tau) = \sum_{k=0}^{\infty} a_{c,j,k} L_k(\tau)$. The exact projection coefficients $a_{c,j,k}$ are obtained via the inner product:
\begin{equation}
a_{c,j,k} = \frac{2k+1}{2} \int_{-1}^{1} h_{c,j}(\tau) L_k(\tau) d\tau.
\label{eq:coeff}
\end{equation}

In practice, to maintain a computationally tractable, fixed-size representation for the neural network, we restrict our model to a $K$-dimensional subspace by computing a truncated Legendre expansion up to degree $K-1$. Crucially, the Hilbert projection theorem guarantees that this specific orthogonal projection yields the optimal approximation of the true continuous signal $h_{c,j}(\tau)$ within this subspace, minimizing the $L^2$ reconstruction error. 
A common practice is to select a value of $K$ value that supports reliable performance at minimal overhead, beyond which improvement is marginal.
Physically, in this optimal truncated representation, lower-order terms (e.g., $k=0, 1$) capture coarse temporal trends such as constant presence or linear drift, while higher-order terms encode fine-grained, non-linear variations in the event features without introducing arbitrary temporal binning.

In practice, the analytic function $h_{c,j}(\tau)$ is unobservable; we only have access to irregular, sparse samples $\{H_{c,j,n}\}$ at discrete timestamps $\{\tau_{j,n}\}$. Therefore, the continuous integral must be approximated numerically. Because event timestamps are generally non-uniformly distributed, naive mean or max pooling produces a highly biased estimate. To obtain a consistent, unbiased estimate, we approximate the temporal integral using numerical quadrature adapted for non-uniform temporal grids. Specifically, we utilize the trapezoidal rule \cite{press2007numerical, davis2007methods}, which assigns a temporal support weight $w_{j,n}$ to each event based on its local sparsity:
\begin{equation}
\label{eq:trape_weights}
w_{j,n} =
\begin{cases}
\frac{1}{2}\big(\tau_{j,2} - \tau_{j,1}\big), & n=1, \\
\frac{1}{2}\big(\tau_{j,n+1} - \tau_{j,n-1}\big), & 1<n<N, \\
\frac{1}{2}\big(\tau_{j,N} - \tau_{j,N-1}\big), & n=N.
\end{cases}
\end{equation}
Under this formulation, isolated events correctly represent larger temporal chunks, while tightly clustered events receive smaller individual weights, preventing rapid event bursts from disproportionately dominating the integral.

Using these weights, we compute the quadrature-corrected empirical moments $z_{c,j,k}$, which estimate the projection of the discrete event features onto the $k$-th Legendre polynomial:
\begin{equation}
\label{eq:moment}
z_{c,j,k}
=
\frac{1}{W_j}
\sum_{n=1}^{N}
m_{j,n}\, w_{j,n}\, H_{c,j,n}\, L_k(\tau_{j,n})
\qquad
W_j = \sum_{n=1}^{N} m_{j,n} w_{j,n}
\end{equation}
where $m_{j,n}\in\{0,1\}$ acts as a validity mask for zero-padded events in pillars with fewer than $N$ entries. Note that the constant Legendre scaling factor $\frac{2k+1}{2}$ from the exact analytical projection is omitted in this estimator, as it is naturally absorbed by the learnable linear coefficients in the subsequent layer. The normalization by $W_j$ ensures that the encoded features remain invariant to the total temporal active duration of the pillar.

\begin{algorithm}[htb]
\caption{Pillar Encoding (PE)}
\label{alg:pillar_encoding}
\textbf{Input:} Raw event set $E$ from temporal window $[t_1, t_2]$, spatial grid size $H \times W$ \\
\textbf{Parameters:} Max pillars $P$, max events $N$, polynomial degrees $K$, latent dims $C$ \\
\textbf{Output:} Dense pseudo-image tensor $\mathbf{I} 
\in \mathbb{R}^{C \times H \times W}$
\begin{algorithmic}[1]
\State Discretize $E$ into spatial pillars. Retain up to $P$ non-empty pillars.
\For{each pillar $j \in \{1, \dots, P\}$}
    \State Subsample events, or zero-pad events and append them to exactly $N$. 
    \State Set $m_{j,n} = 0$ if event $n$ is zero padded; otherwise, $m_{j,n} = 1$ (real event).
    \State Normalize absolute timestamps to $\tau_{j,n} \in [-1, 1]$.
    \State Sort intra-pillar events such that $\tau_{j,1} \le \tau_{j,2} \dots \le \tau_{j,N}$ for each event $n$ with $m_{j,n}$ = 1.
    \State Compute temporal and spatial offsets (mean-centered and spatial/temporal-centered).
    \State Construct the $D$-dimensional augmented feature vectors $l_{j,n}$.
    \State Compute non-uniform trapezoidal weights $w_{j,n}$ based on local sparsity (Eq.~\ref{eq:trape_weights}).
    \State Compute pillar weight normalizer: $W_j \gets \sum_{n=1}^{N} m_{j,n}\, w_{j,n}$
\EndFor
\State Assemble tensor $\mathbf{X} \in \mathbb{R}^{D \times P \times N}$ from $\{l_{j,n}\}$.
\State Extract latent features via shared MLP: $\mathbf{H} \gets \sigma\!\left(\mathrm{BN}\!\left(\mathbf{W}\mathbf{X}+\mathbf{b}\right)\right)$
\For{each pillar $j \in \{1, \dots, P\}$ and channel $c \in \{1, \dots, C\}$}
    \For{degree $k \in \{0, \dots, K-1\}$}
        \State Compute quadrature-corrected empirical moment (Eq.~\ref{eq:moment}):
        \State $z_{c,j,k} \gets \frac{1}{W_j} \sum_{n=1}^{N} m_{j,n}\, w_{j,n}\, H_{c,j,n}\, L_k(\tau_{j,n})$
    \EndFor
    \State Linearly combine temporal moments using learnable weights (Eq.~\ref{eq:represent}): 
    \State $r_{c,j} \gets \sum_{k=0}^{K-1} \alpha_{c,k}\, z_{c,j,k} + \beta_c$
\EndFor
\State Scatter pillar representations $\mathbf{R} \in \mathbb{R}^{C \times P}$ back to the $H \times W$ spatial grid to form $\mathbf{I}$.
\State \Return $\mathbf{I}$
\end{algorithmic}
\end{algorithm}

%\paragraph{Learnable Aggregation.} 
\textbf{Learnable Aggregation.}
Finally, we linearly combine the $K$ temporal moments using learnable per-channel coefficients to form the final pillar representation:
\begin{equation}
r_{c,j}
=
\sum_{k=0}^{K-1}\alpha_{c,k}\,z_{c,j,k}+\beta_c
\qquad
\mathbf{R}\in\mathbb{R}^{C\times P}
\label{eq:represent}
\end{equation}
where $\boldsymbol{\alpha} \in \mathbb{R}^{C\times K}$ and $\boldsymbol{\beta} \in \mathbb{R}^{C}$ are trainable parameters. These coefficients introduce only $C(K+1)$ additional parameters, making the encoding highly efficient. The parameters $C$ and $K$ act as tunable hyperparameters.

Once encoded, the features are scattered back to their original spatial locations to create a pseudo-image of size $(C, H, W)$. This pillar-based continuous-time representation enables efficient spatial feature learning while retaining the rich, asynchronous temporal dynamics of the raw event stream. 
At high inference frequencies (e.g., 200 Hz), $\Delta t$ is sufficiently small that cross-pillar motion blur is minimized; residual spatial trajectories are effectively aggregated by the receptive field of the downstream backbone.

% \begin{theorem}[Optimality of Legendre Projection]\label{thm:legendre}
% Let $h \in L^2([-1,1])$, and let $\mathcal{S}_K = \mathrm{span}\{L_0, L_1, \dots, L_{K-1}\}$, where $\{L_k\}_{k=0}^{K-1}$ are the Legendre polynomials, orthogonal with respect to the $L^2([-1,1])$ inner product. Define
% $$
% h_K(\tau) = \sum_{k=0}^{K-1} a_k L_k(\tau), \quad \text{with} \quad a_k = \frac{2k+1}{2} \langle h, L_k \rangle.
% $$
% Then $h_K$ is the unique orthogonal projection of $h$ onto $\mathcal{S}_K$, and hence $h_K = \arg\min_{g \in \mathcal{S}_K} \| h - g \|_{L^2}^2$.
% \end{theorem}
% %\begin{proof}
% %Proof provided in Appendix~\ref{sec:theory}.
% %\end{proof}

% \begin{theorem}[Bias of Naive Aggregation vs. Quadrature]\label{thm:quadruture}
% Let $\{\tau_n\}_{n=1}^N$ be an ordered sequence of timestamps with $\tau_1 < \cdots < \tau_N$, sampled according to an inhomogeneous temporal density $\rho(\tau)$. Define $H_n := h(\tau_n)$, where $h : [-1,1] \to \mathbb{R}$ is the underlying continuous-time signal. Let $\Delta \tau_n := \tfrac{1}{2}(\tau_{n+1} - \tau_{n-1})$ denote the local temporal spacing for interior points, with appropriate one-sided definitions at the boundaries.
% The naive empirical mean estimator $\hat{I}_k^{\mathrm{naive}}$ is biased by $\rho(\tau)$, whereas the quadrature estimator $\hat{I}_k^{\mathrm{quad}}$ asymptotically recovers the unscaled projection integral $I_k$.
% \end{theorem}

% %\begin{proof}
% %Proof provided in Appendix~\ref{sec:theory}.
% %\end{proof}

Our PE formulation is grounded in two theoretical guarantees (proofs provided in Appendix~\ref{sec:theory}). First, by defining the coefficients as in Eq.~\ref{eq:coeff}, our truncated Legendre series acts as the unique orthogonal projection of the underlying signal $h_{c,j}(\tau)$ for all $c, j$ at $\tau$ onto the subspace of Legendre polynomials, minimizing the $L^2$ reconstruction error. 
Second, our quadrature estimator in Eqs.~\ref{eq:trape_weights} and~\ref{eq:moment}
asymptotically recovers the true projection integral. Our pseudocode is detailed in Algorithm~\ref{alg:pillar_encoding}.

%A theoretical justification of PE is given in Appendix~\ref{sec:theory}.

%The pseudocode of our PE method is given in Algorithm~\ref{alg:pillar_encoding}.

\subsection{Frequency-Aware Training}
\label{sec:fat}
We divide our training strategy into two sequential phases. First, we obtain dense supervision at multiple temporal frequencies. Second, we utilize a soft mean-teacher curriculum to enforce representational robustness across these frequencies.

%\paragraph{Phase 1: Multi-Frequency Bounding Box Generation.}
%\label{para:fat_phase1}
\textbf{Phase 1: Multi-Frequency Bounding Box Generation.}
The first phase aims to construct dense, frequency-specific annotations from sparsely labeled data. Let $\mathcal{E}$ denote the event stream, and $\{(t_k,\mathcal{B}_{t_k})\}$ the sparse ground-truth annotations available only at canonical timestamps $t_k$, which typically correspond to a canonical frequency $f_c$ (e.g., $20$\,Hz). Because annotations are not available at higher temporal frequencies, we first train an event detector, e.g., EvRT-DETR \cite{torbunovEvRTDETRLatentSpace}, augmented by the proposed PE using the original labeled data in the canonical setting.

Based on this PE-augmented detector, we generate dense annotations at target frequencies $f \in \mathcal{F}=\{f_{\min},\dots,f_{\max}\}$ by running inference at interpolated timestamps in a sliding-window manner. For a target frequency $f$, the detector is applied at uniformly spaced inference timestamps $t_k = k \cdot \Delta t_f, \text{where }  \Delta t_f = 1/f$. Each input is constructed using a fixed canonical accumulation window $\Delta t_c = 1/f_c$ over the interval $[t_k-\Delta t_c,\,t_k]$, effectively decoupling the temporal sampling rate from the event representation. We generate annotations progressively (e.g., first at 40\,Hz, scaling up to 200\,Hz), consistently maintaining temporal association through tracking-by-detection~\cite{bewleySimpleOnlineRealtime2016}.

% Based on this PE-augmented detector, we generate dense annotations at target frequencies $f \in \mathcal{F}=\{f_{\min},\dots,f_{\max}\}$ by running inference at interpolated timestamps in a sliding-window manner. 
% Crucially, inference is performed using the canonical event representation but evaluated at denser timestamps associated with each target frequency. 
% For a target frequency $f$, the detector is applied at timestamps spaced by $\Delta t_f = 1/f$, where each input is formed by accumulating events over the interval $[t-\Delta t_c,\,t]$ using a canonical accumulation window $\Delta t_c = 1/f_c$. We generate annotations progressively (e.g., first at $40$\,Hz, scaling up to $200$\,Hz), consistently maintaining temporal association through tracking-by-detection~\cite{bewleySimpleOnlineRealtime2016}.

To ensure high-quality supervision, we adopt decoupled thresholds: a lower tracking threshold (e.g., 0.3) to favor recall by reducing missed detections, and a higher detection threshold (e.g., 0.6) to promote precision by limiting false positives (see Appendices~\ref{appendix:result_details} and~\ref{app:reproduce} for details). We then associate detections across time and discard short tracklets as spurious. To account for denser sampling at higher frequencies, the minimum tracklet length is increased proportionally with $f$. The remaining reliable trajectories are used to interpolate annotations at timestamps with missing detections.

The resulting dense, frequency-specific training set is $\tilde{\mathcal{D}}_f = \{(\mathbf{r}_{f,i}, \tilde{\mathcal{B}}_{f,i})\}_{i=1}^{M_f}$, where $\mathbf{r}_{f,i}$ denotes an event representation---PE output---at frequency $f$, and $\tilde{\mathcal{B}}_{f,i}$ the corresponding supervision, consisting of either ground-truth annotations or labels interpolated along tracked trajectories. Notably, these dense pseudo labels (bounding boxes) are used to enable offline FAT, but are not utilized during online inference in FATE.

\textbf{Phase 2: Soft Mean-Teacher Self-Training.}
%In the second phase, we use the dense multi-frequency supervision constructed in Phase 1 to train a detector robust to varying temporal resolutions. 
The proposed frequency-aware self-training adopts a soft mean-teacher framework~\cite{xuEndToEndSoftTeacher2021, anttimeanTeacher2017}. In particular, it exploits three sources of supervision: (i) ground-truth annotations at the originally labeled timestamps, (ii) interpolated labels from Phase 1, and (iii) pseudo-labels produced by the teacher network during self-training. The first two provide explicit supervision, while the third provides a teacher–student consistency constraint that regularizes the student model.

Let $\mathbf{R}_f$ and $\mathbf{R}_c$ denote the sets of event representations at frequencies $f$ and $f_c$ (canonical frequency). We write $\mathbf{r}_f \in \mathbf{R}_f$ and $\mathbf{r}_c \in \mathbf{R}_c$ for individual samples. The teacher network processes $\mathbf{r}_c$, while the student network processes $\mathbf{r}_f$.
%We instantiate a student network $S_{\theta_s}$ and a teacher network $T_{\theta_t}$. 
Given inputs $\mathbf{r}_f$ and $\mathbf{r}_c$, a student network $S_{\theta_s}$ and a teacher network $T_{\theta_t}$ output predictions:
\begin{equation}
\left\{(s_{s,j}, b_{s,j})\right\}_{j=1}^{N_s}
= S_{\theta_s}(\mathbf{r}_f),
\qquad
\left\{(s_{t,j}, b_{t,j})\right\}_{j=1}^{N_t}
= T_{\theta_t}(\mathbf{r}_c).
\end{equation}
Instead of uniformly sampling frequencies, we employ a linear curriculum schedule over the frequency-specific datasets. At training round $r$, the student frequency is sampled according to $f \sim p_r(f)$, which gradually shifts the probability mass toward higher, sparser frequencies as training progresses.

For bounding box regression, the DETR \cite{detr} head uses a weighted combination of $\ell_1$ and GIoU losses, while the YOLOX \cite{yolox2021} head primarily uses the standard IoU regression loss:
\begin{equation}
\label{eq:box-loss}
   \mathcal{L}_{\mathrm{box}}(b,\hat b) = \lambda_{\ell_1}\|b-\hat b\|_1 + \lambda_{\mathrm{iou}} \mathcal{L}_{\mathrm{GIoU}}(b,\hat b)
\end{equation}

For the DETR head, let $\mathcal{M}_i$ denote the set of matched prediction--target pairs for sample $i$, established via bipartite matching. While the regression loss $\mathcal{L}_{\mathrm{box}}$ is evaluated exclusively over this matched set $\mathcal{M}_i$, the classification loss explicitly penalizes all remaining mismatched predictions as a "no object" ($\emptyset$) background class, a mechanism crucial for suppressing false positive detections.

Each target is represented as a pair $(c_v, b_v)$, where $c_v$ denotes the class label and $b_v \in \mathbb{R}^4$ the corresponding ground-truth bounding box. 
%The targets are drawn from the dense dataset produced in Phase 1. 
Given that, the detection objective is formulated as:
% \begin{equation}
% \mathcal{L}_{\mathrm{det}}
% =
% \frac{1}{|\mathcal{U}|}
% \sum_{i\in\mathcal{U}}
% \sum_{(u,v)\in \mathcal{M}_i}
% w_v
% \left[
% \ell_{\mathrm{cls}}(s_{s,u}, c_v)
% +
% \ell_{\mathrm{box}}(b_{s,u}, b_v)
% \right],
% \end{equation}
\begin{equation}
\mathcal{L}_{\mathrm{det}}
=
\frac{1}{|\mathcal{U}|}
\sum_{i\in\mathcal{U}}
\left[
\sum_{(u,v)\in \mathcal{M}_i}
w_v
\left(
\ell_{\mathrm{cls}}(s_{s,u}, c_v)
+
\ell_{\mathrm{box}}(b_{s,u}, b_v)
\right)
+
\sum_{u\notin\mathcal{M}_i}
\ell_{\mathrm{cls}}(s_{s,u}, \emptyset)
\right]
\end{equation}

where $\mathcal{U}$ is the mini-batch. 
The supervision weight $w_v$ distinguishes ground-truth labels from generated labels:
\begin{equation}
w_v =
\begin{cases}
1, & \text{if } (c_v,b_v) \text{ is a ground-truth annotation;} \\
p_v, & \text{if } (c_v,b_v) \text{ is a label generated in Phase 1,}
\end{cases}
\label{eq:w_v}
\end{equation}
where $p_v \in [0,1]$ denotes the scalar confidence assigned by the teacher to the matched generated label.

Additionally, we impose a consistency loss between the teacher's canonical predictions and the student's high-frequency predictions. The teacher network provides pseudo-labels during training, which serve as soft targets for regularizing the student model.
We use Kullback--Leibler (KL) divergence for classification consistency and $\ell_1$ loss for localization consistency:
\begin{equation}
\mathcal{L}_{\mathrm{cons}}
=
\frac{1}{|\mathcal{U}|}
\sum_{i\in\mathcal{U}}
\frac{1}{|\mathcal{A}_i|}
\sum_{(u,v)\in\mathcal{A}_i}
\left[
\mathrm{KL}(q_{t,v}\,\|\,q_{s,u})
+
\|b_{t,v} - b_{s,u}\|_1
\right],
\end{equation}
where $\mathcal{A}_i$ denotes the set of matched teacher--student prediction pairs for sample $i$ obtained via Hungarian matching~\cite{detr}. Moreover, $q_{t,v}$ and $q_{s,u}$ (respectively, $b_{t,v}$ and $b_{s,u}$) denote the class probability distributions over $K$ object categories (bounding box predictions) produced by the teacher and student networks.

The student weights $\theta_s$ are updated via gradient descent, while the teacher weights $\theta_t$ are updated using an exponential moving average (EMA):
\begin{equation}
\theta_t \leftarrow \gamma \theta_t + (1-\gamma)\theta_s, \qquad \gamma \in [0,1).
\end{equation}
Rather than strictly inheriting canonical-frequency blind spots, our temporal tracking interpolation actively recovers missing intermediate detections, while the EMA-based curriculum mitigates the risk of the student overfitting to isolated base-detector errors.

%% file: sections/4.eval5.tex
We evaluate FATE against recent event-based object detection methods on two standard benchmarks, analyzing performance across a range of temporal frequencies.

\subsection{Evaluation Setup}

%\paragraph{Datasets.}
\textbf{Datasets.} Experiments are conducted on Gen1~\cite{tournemireLargeScaleEventbased2020} and 1Mpx~\cite{perotLearningDetectObjects}, which differ substantially in spatial resolution and event density. This setup enables evaluation under both sparse and high-resolution regimes, assessing robustness and scalability. Gen1 provides sparse annotations from first-generation sensors (1--4,Hz), whereas 1Mpx leverages fourth-generation sensors with significantly higher temporal density ($\sim$60 Hz). To study frequency generalization, we evaluate models at the canonical 50 ms window (20 Hz), as well as under shorter temporal windows corresponding to higher frequencies from 40 Hz to 200 Hz (25--5 ms).

\textbf{Baselines and Implementation Details.} We compare FATE against four state-of-the-art baselines: RVT-B \cite{gehrigRecurrentVisionTransformers2023}, GET \cite{pengGETGroupEvent2023}, S5-ViT-B \cite{zubicStateSpaceModels2024}, and EvRT-DETR-B \cite{torbunovEvRTDETRLatentSpace}, spanning convolutional, recurrent, transformer, and state-space paradigms. To evaluate mAP, we integrate FATE into two representative detectors: S5-ViT-B (S5-ViT backbone, YOLOX head) and EvRT-DETR (ResNet-50 backbone, recurrent transformer). Baselines use official checkpoints for fairness. See Appendix~\ref{app:reproduce} for full implementation details.

\subsection{Comparison to State-of-the-Art}
%\input{tables/2}

% Gen 1 Table with Ground Truth Annotations
\begin{table}[htb]
\vspace{-0.6cm}
\centering
\caption{mAP at different operating frequencies on the Gen1 test set}
\label{tab:freq_scaling_gen1}
\setlength{\tabcolsep}{4pt}
\scriptsize
\begin{tabular}{l|ccccc}
\toprule
& \multicolumn{5}{c}{Frequency (Hz)} \\
\cmidrule(lr){2-6}
Model & 20 Hz & 40 Hz & 80 Hz & 100 Hz & 200 Hz \\
\midrule
RVT-B \cite{gehrigRecurrentVisionTransformers2023} & 47.20 & 35.13 & 21.98 & 18.61 & 8.35 \\
GET \cite{pengGETGroupEvent2023} & 47.90 & 34.15 & 19.97 & 15.13 & 5.35 \\
S5-ViT-B \cite{zubicStateSpaceModels2024}  & 47.40 & 46.44 & 45.08 & 42.49 & 39.84 \\
EvRT-DETR-B \cite{torbunovEvRTDETRLatentSpace} & \underline{52.70} & \underline{50.60} & 45.00 & 42.20 & 32.00 \\
\midrule

\textbf{\name{}-S (PE + FAT + S5-ViT-B)} & 48.20 & 46.90 & \underline{47.02} & \underline{43.18} & \underline{40.80} \\

\textbf{\name{}-E (PE + FAT + EvRT-DETR-B)} & \textbf{53.14} & \textbf{51.60} & \textbf{48.40} & \textbf{46.80} & \textbf{42.38} \\

\bottomrule
\end{tabular}
\vspace{-0.35cm}
\end{table}

% 1Mpx Ground Truth Table
\begin{table}[htb]
\vspace{-0.25cm}
\centering
\caption{mAP at different operating frequencies on the 1Mpx test set}
\label{tab:freq_scaling_1mpx}
\setlength{\tabcolsep}{4pt}
\scriptsize
\begin{tabular}{l|ccccc}
\toprule
& \multicolumn{5}{c}{Frequency (Hz)} \\
\cmidrule(lr){2-6}
Model & 20 Hz & 40 Hz & 80 Hz & 100 Hz & 200 Hz \\
\midrule
RVT-B \cite{gehrigRecurrentVisionTransformers2023} & 47.40 & 42.51 & 33.20 & 30.29 & 16.36 \\
GET \cite{pengGETGroupEvent2023}  & 48.40 & 40.51 & 30.30 & 28.11 & 15.44 \\
S5-ViT-B \cite{zubicStateSpaceModels2024} & 47.20 & 46.49 & 46.11 & 45.80 & 39.70 \\
EvRT-DETR-B \cite{torbunovEvRTDETRLatentSpace} & \underline{50.40} & 48.20 & 43.10 & 41.00 & 33.40 \\
\midrule

\textbf{\name{}-S (PE + FAT + S5-ViT-B)} & 48.80 & \underline{48.26} & \textbf{47.30} & \textbf{46.14} & \underline{40.20} \\

\textbf{\name{}-E (PE + FAT + EvRT-DETR-B)}  & \textbf{51.04} & \textbf{48.90} & \underline{46.32} & \underline{45.77} & \textbf{40.54}  \\

\bottomrule
\end{tabular}
\vspace{-0.3cm}
\end{table}

As frequency increases, all methods degrade due to extreme sparsity from shorter accumulation windows (Tables~\ref{tab:freq_scaling_gen1} and~\ref{tab:freq_scaling_1mpx}). 
However, FATE consistently outperforms prior work, with gains that become more pronounced at higher frequencies. At 20\,Hz, \textbf{\name{}-E} achieves 53.14 mAP on Gen1 and 51.04 mAP on 1Mpx, improving over EvRT-DETR-B by +0.44 and +0.64 mAP. Under the most challenging 200\,Hz setting, the margin widens substantially: \name{}-E exceeds the strongest high-frequency baseline (S5-ViT-B) by +2.54 mAP on Gen1 and +0.84 mAP on 1Mpx, demonstrating strong robustness under extreme sparsity.

Importantly, these gains generalize across different architectural paradigms. Integrating FATE into S5-ViT-B (FATE-S) also yields consistent improvements over S5-ViT-B, particularly in high-frequency regimes (+0.96 mAP on Gen1 and +0.50 mAP on 1Mpx at 200 Hz). This confirms that FATE's benefits are not tied to a specific detector design, but rather address fundamental challenges in event-based representation and supervision-inference frequency mismatch.

Qualitative results (Appendix~\ref{app:qual}) show that under short temporal windows, FATE maintains stable predictions—especially for small or weak targets—avoiding the missed detections of competing methods. These gains incur negligible overhead (+0.1 ms latency, less than 0.15 M parameters; Appendix~\ref{appendix:result_details}). This robustness stems from: (i) continuous-time polynomial pillar encoding, which preserves fine-grained temporal dynamics under sparsity, and (ii) frequency-aware training, which mitigates train--test mismatch across temporal resolutions.

% "architecture-agnostic" can be too bold.
%Importantly, these gains are architecture-agnostic. Applying FATE to S5-ViT-B (\textbf{\name{}-S}) yields consistent improvements over its backbone, particularly at high frequencies (+0.96 mAP on Gen1 and +0.50 mAP on 1Mpx at 200\,Hz), confirming that the proposed components generalize across detector designs.

% Qualitative results (Appendix~\ref{app:qual}) further highlight this advantage: under short temporal windows, competing methods exhibit missed detections and temporal instability, whereas FATE maintains stable predictions, particularly for small or weakly activated targets. These gains come with minimal overhead—only 0.1\,ms additional inference latency and minimal parameters for the proposed pillar encoding (Appendix~\ref{appendix:result_details}). We attribute this robustness to two key factors: (i) pillar encoding via a continuous-time polynomial approximation, which preserves fine-grained temporal structure under sparse observations, and (ii) frequency-aware training, which mitigates train--test mismatch by enforcing consistency across temporal resolutions.

%% file: sections/6.ablation2.tex
%Tables~\ref{tab:ablation_components_fate_e} and \ref{tab:ablation_components} analyze the contributions of PE and FAT
%, and their key design parameters 
%on the Gen1 benchmark.

%\paragraph{Effect of PE and FAT.}

% Tables~\ref{tab:ablation_components_fate_e} and \ref{tab:ablation_components} analyze the contributions of PE and FAT
% on the Gen1 benchmark.
% The tables show that both PE and FAT are critical for high-frequency performance, where short accumulation windows produce highly sparse event streams. 
% PE outperforms EvRT-DETR-B and S5-ViT-B in 6/10 settings (e.g., +4.40 mAP at 200\,Hz over EvRT-DETR-B) and remains competitive otherwise, demonstrating its ability to preserve intra-window temporal structure. FAT delivers even larger gains, outperforming baselines in 9/10 settings (e.g., +9.10 mAP at 200\,Hz over EvRT-DETR-B), indicating that multi-frequency supervision effectively reduces the train--test mismatch without increasing model capacity. 

Tables~\ref{tab:ablation_components_fate_e} and \ref{tab:ablation_components} ablate PE and FAT on Gen1, confirming both are critical for sparse, high-frequency inference. PE outperforms baselines in 6/10 settings (e.g., +4.40 mAP at 200 Hz over EvRT-DETR-B) by preserving intra-window temporal dynamics. FAT yields even broader gains, leading in 9/10 settings (e.g., +9.10 mAP at 200,Hz over EvRT-DETR-B), proving that multi-frequency supervision mitigates train--test mismatch without adding capacity overhead.

Combined, PE and FAT achieve the best overall performance, supporting our central claim that robust temporal scaling requires both continuous representations and frequency-aligned supervision. Additional ablations on frequency-aware training strategies, polynomial degree $K$, pillar size, and temporal capacity $C$ are provided in Appendix~\ref{appendix:result_details}.

%provides consistent gains at 80--200\,Hz (e.g., +4.40 mAP at 200\,Hz) by preserving intra-window temporal structure without degrading standard-frequency performance. 

\noindent \textbf{Summary.} PE alleviates the representational bottleneck of short event windows, FAT mitigates cross-frequency supervision mismatch, and their combination enables strong performance across frequencies.

\input{tables/3}

%% file: tables/3.tex
% Ablation Gen1 (DONE)
\begin{table*}[htb]
\vspace{-0.2cm}
\centering
\caption{Ablation of the FATE-E (PE + FAT + EvRT-DETR-B) components on the Gen1 test set} 
\label{tab:ablation_components_fate_e}
\setlength{\tabcolsep}{4pt}
\scriptsize
\begin{tabular}{lccccc c}
\toprule
\multirow{2}{*}{Setting} 
& \multicolumn{5}{c}{mAP (\%) $\uparrow$}
& \multicolumn{1}{c}{Model Size} \\
\cmidrule(lr){2-6} \cmidrule(lr){7-7}
& 20\,Hz & 40\,Hz & 80\,Hz & 100\,Hz & 200\,Hz
& \#Params (M) \\
\midrule

EvRT-DETR-B 
& 52.70 & 50.60 & 45.00 & 42.20 & 32.00 
& 57.14 \\

+ PE  
& 51.94 & 50.31 & 46.06 & 44.10 & 36.40 
& 57.16 \\

+ FAT 
& \underline{52.36} & \underline{51.24} & \textbf{48.50} & \underline{46.48} & \underline{41.10} 
& 57.14 \\

+ PE + FAT 
& \textbf{53.14} & \textbf{51.60} & \underline{48.40} & \textbf{46.80} & \textbf{42.38}  
& 57.16 \\

\bottomrule
\end{tabular}
%\vspace{-0.2cm}
\end{table*}

% \Comment{There was an bug in calculating the params. Updated it}

% Ablation Gen1 (DONE)
\begin{table*}[htb]
\vspace{-0.25cm}
\centering
\caption{Ablation of the FATE-S (PE + FAT + S5-ViT-B) components on the Gen1 test set} 
\label{tab:ablation_components}
\setlength{\tabcolsep}{4pt}
\scriptsize
\begin{tabular}{lccccc c}
\toprule
\multirow{2}{*}{Setting} 
& \multicolumn{5}{c}{mAP (\%) $\uparrow$}
& \multicolumn{1}{c}{Model Size} \\
\cmidrule(lr){2-6} \cmidrule(lr){7-7}
& 20\,Hz & 40\,Hz & 80\,Hz & 100\,Hz & 200\,Hz
& \#Params (M) \\
\midrule

S5-ViT-B 
& 47.40 & 46.44 & 45.08 & 42.49 & 39.84 
& 18.19 \\

+ PE  
& 47.70 & 46.10 & 45.52 & 43.08 & 39.30 
& 18.33 \\

+ FAT 
& \underline{48.08} & \underline{46.78} & \underline{46.50} & \textbf{43.24} & \underline{40.38} 
& 18.19 \\

+ PE + FAT 
& \textbf{48.20} & \textbf{46.90} & \textbf{47.02} & \underline{43.18} & \textbf{40.80}
& 18.33 \\

\bottomrule
\end{tabular}
\vspace{-0.4cm}
\end{table*}

%% file: sections/7.conclusion.tex
To alleviate a fundamental limitation in event-based object detection---performance degradation at high operating frequencies due to event sparsity---we introduced the FATE framework, uniting Pillar Encoding (PE) with Frequency-Aware Training (FAT). PE provides a continuous-time, $L^2$-optimal polynomial representation that retains fine-grained temporal dynamics while avoiding internal temporal sub-binning. Concurrently, FAT mitigates the train--test frequency mismatch by leveraging temporally dense pseudo-labels to promote robustness across the frequency spectrum. By generalizing across different architectural paradigms, FATE consistently outperforms strong baselines, with higher performance gains under high-frequency conditions (up to 200 Hz). Overall, our results underscore the critical importance of continuous-time modeling and cross-frequency supervision for robust, high-speed event-based perception.

%% file: sections/8.appendix.tex
\section{Theoretical Justifications for Pillar Encoding}
\label{sec:theory}
\input{sections/justify3}

\section{Additional Quantitative Results}
\label{appendix:result_details}

\subsection{Detection Performance Preservation}

Figure~\ref{fig:freq-pair} shows the mAP of S5-ViT-B, EvRT-DETR-B, FATE-S, and FATE-E on Gen1, highlighting how well each method preserves performance as frequency increases from 20\,Hz to 200\,Hz. Both FATE-E and FATE-S consistently outperform their respective backbones, with gains that grow at higher frequencies. 

While FATE-E achieves stronger overall performance, the gap between FATE-E and FATE-S narrows as frequency increases. This trend reflects the advantage of state-space modeling in S5-ViT-B, which captures temporal dynamics more effectively under extreme sparsity.

\begin{figure}[h]
  \centering
  \includegraphics[width=0.5\textwidth]{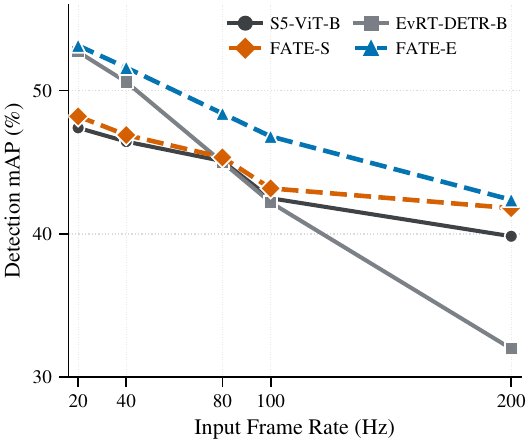}
  \caption{Comparison of mAP across different operating frequencies on Gen1.}
  \label{fig:freq-pair}
\end{figure}

\subsection{Effect of Dense Temporal Supervision}

We randomly sample 10,000 predicted bounding boxes from the Gen1 test set using two baseline detectors augmented with PE: PE + S5-ViT-B and PE + EvRT-DETR-B, both operating at 20\,Hz. Figure~\ref{fig:pr-thresholds-fate} shows the precision–recall trade-off for the \textit{car} and \textit{pedestrian} classes. Pseudo-labels for cars consistently achieve higher precision and recall than those for pedestrians, likely due to the more consistent geometric structure of vehicles.

\paragraph{Decoupled Thresholds.} To ensure high-quality supervision, we adopt decoupled thresholds: a lower tracking threshold of 0.3 to favor recall (reducing missed detections), and a higher detection threshold of 0.6 to promote precision (limiting false positives). As shown in Figure~\ref{fig:pr-thresholds-fate}, these choices provide a consistent balance of high recall and precision, respectively.

\paragraph{Pseudo Label Quality.}
Table~\ref{tab:pseudo_quality_class} provides a detailed evaluation of pseudo-label quality across frequencies from 20\,Hz to 200\,Hz, revealing three key insights for downstream supervision.

\begin{itemize}
    \item \textit{Temporal sparsity and performance decay:}
    All metrics---precision, recall, and localization accuracy (Median IoU)---degrade monotonically as frequency increases. At 200\,Hz, the shortened accumulation window leads to sparse event representations, making it difficult to recover coherent object geometry. For example, the proportion of boxes with $\text{IoU} > 0.5$ for cars drops from 83.0\% at 20\,Hz to 65.2\% at 200\,Hz for PE + EvRT-DETR-B, highlighting a trade-off between temporal resolution and label reliability.

    \item \textit{Architectural robustness:}
    Across all frequencies, PE + EvRT-DETR-B outperforms PE + S5-ViT-B, in line with the EvRT-DETR-B and S5-ViT-B comparisons in the main text. The DETR-based model maintains stronger spatial alignment, achieving a median IoU $\geq 0.5$ for cars even at 200 Hz, whereas the ViT-based baseline drops to 0.47, indicating greater robustness to noise and sparsity in high-frequency event data. Nevertheless, FATE-E and FATE-S, which adopt PE augmented by FAT, improve EvRT-DETR-B and S5-ViT-B, respectively, demonstrating consistent gains across heterogeneous architectures.
    
    % Across all frequencies, PE + EvRT-DETR-B consistently outperforms PE + S5-ViT-B, conforming to the results in the main text. The DETR-based model maintains stronger spatial alignment, achieving a Median IoU $\geq 0.50$ for cars even at 200 Hz, whereas the ViT-based baseline drops to 0.47. This suggests greater robustness to the noise and sparsity in high-frequency event data. Independent of individual backbones, FATE-E and FATE-S enhances the performance of EvRT-DETR-B and S5-ViT-B with heterogeneous architectures, respectively.

    \item \textit{Class-wise variance (cars vs. pedestrians):}
    A consistent gap exists between object categories. Pseudo-labels for cars exhibit higher quality than those for pedestrians (e.g., 0.83 vs.\ 0.72 precision at 20\,Hz), likely due to larger spatial extent and more rigid motion. In contrast, pedestrians suffer from fragmentation effects at high frequencies, with recall decreasing to 0.32–0.36 at 200 Hz.
    
\end{itemize}

\paragraph{Training Strategy Ablations.}
Furthermore, we compare FAT to three ablated training strategies across frequencies:
\begin{itemize}
    \item \textit{Multi-frequency training (no additional supervision):} The detector is trained across multiple accumulation windows (20--200\,Hz) using only sparse ground-truth annotations. For each frequency, the nearest labels are reused without introducing intermediate supervision, isolating the effect of multi-frequency exposure.

    \item \textit{Consistency-only training (no densification):} A consistency loss is enforced between predictions at different frequencies, while supervision remains limited to sparse ground-truth annotations, without generating dense labels.

    \item \textit{Naive annotation interpolation:} Dense supervision is introduced by linearly interpolating ground-truth bounding boxes across timestamps to generate high-frequency labels. These interpolated annotations are used for training at all frequencies, assuming smooth motion.
\end{itemize}

Table~\ref{tab:fat_component_ablation} shows that multi-frequency training improves high-frequency performance over the base detector (e.g., +4.08 mAP at 200 Hz for PE + EvRT-DETR-B) but leads to suboptimal accuracy at the canonical training frequency. Consistency-only training yields comparable results. Introducing dense supervision via naive interpolation improves mAP across all frequencies, motivating FAT, but its effectiveness is limited by noisy supervision (e.g., when adjacent frames do not share the same object instance). FAT achieves the best performance across all frequencies, indicating that integrating multi-frequency training, student-teacher consistency, and pseudo-label generations based on tracking-by-detection is key to robust detection across varying temporal resolutions in event-based settings.

%Table~\ref{tab:fat_component_ablation} shows that training with multiple frequencies improves high-frequency performance compared to the base detector (e.g., +4.08 mAP at 200Hz over PE + EvRT-DETR-B), but results in sub-optimal accuracy at the canonical frequency at which the base detector is trained. Consistency-only training yields comparable performance to multi-frequency training. Introducing dense supervision via naive interpolation improves mAP across all frequencies, which motivates the design of FAT.  However, its effectiveness is limited by noisy supervision (e.g., if two adjacent frames do not share the same object instance, interpolation cannot be defined). FAT achieves the best performance across all frequencies, demonstrating that combining dense supervision with teacher-based pseudo-labeling is essential for robust detection across varying temporal resolutions for event-based object detection.

\begin{figure*}[t]
      \centering
      \begin{subfigure}[t]{0.24\textwidth}
          \centering
          \includegraphics[width=\linewidth]{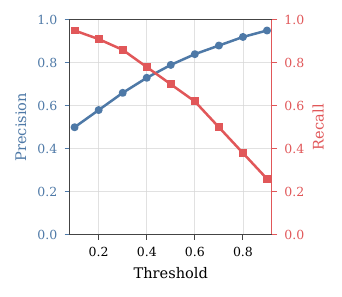}
          \caption{\footnotesize Cars, PE+EvRT-DETR-B}
          \label{fig:pr-fate-e-car}
      \end{subfigure}
      \hfill
      \begin{subfigure}[t]{0.24\textwidth}
          \centering
          \includegraphics[width=\linewidth]{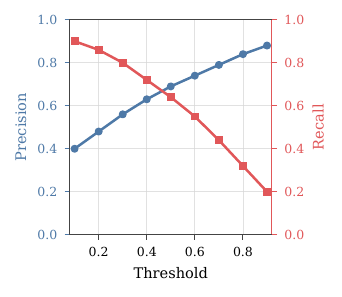}
          \caption{\footnotesize Pedestrians, PE+EvRT-DETR-B}
          \label{fig:pr-fate-e-ped}
      \end{subfigure}
      \hfill
      \begin{subfigure}[t]{0.24\textwidth}
          \centering
          \includegraphics[width=\linewidth]{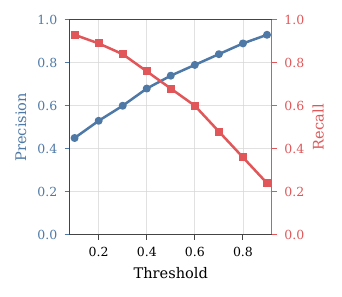}
          \caption{\footnotesize Cars, PE+S5-ViT-B }
          \label{fig:pr-fate-s-car}
      \end{subfigure}
      \hfill
      \begin{subfigure}[t]{0.24\textwidth}
          \centering
          \includegraphics[width=\linewidth]{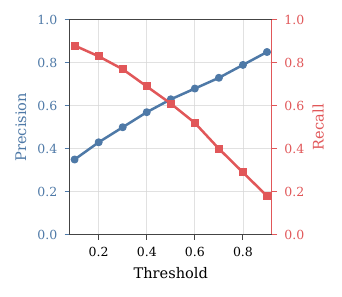}
          \caption{\footnotesize Pedestrians, PE+S5-ViT-B}
          \label{fig:pr-fate-s-ped}
      \end{subfigure}

      \caption{Precision and recall of pseudo-labels for cars and pedestrians on 10K samples randomly drawn from Gen1.} 
      \label{fig:pr-thresholds-fate}
  \end{figure*}

\input{tables/7}
\input{tables/4}

\input{tables/5}

\input{tables/8}

\subsection{Sensitivity Ablations of Pillar Encoding}

\paragraph{Impact of Temporal Channel Dimension.}
Table~\ref{tab:fe_channel_ablation} presents the effect of the temporal channel dimension, $C$, on detection performance across different operating frequencies. Empirically, we observe a consistent performance peak at $C=64$, which yields the highest mAP at 20 Hz, 100 Hz, and 200 Hz among the evaluated configurations. Increasing the capacity from $C=32$ to $C=64$ leads to steady gains, suggesting that sufficient channel width is necessary to capture complex temporal dynamics effectively.

However, further scaling to $C=80$ and $C=96$ results in continuous performance degradation. Given the sparse and noisy nature of event data, excessive temporal capacity may encourage overfitting to localized sensor noise rather than promoting robust, generalizable motion representations. Notably, the configuration at $C=64$ achieves best performance within our tested range while introducing only 0.01M additional parameters compared to the $C=32$ baseline, indicating a favorable trade-off between temporal expressiveness and model efficiency.

\paragraph{Efficiency Ablation}
Table~\ref{tab:efficiency} highlights the efficiency of our proposed approach. Integrating PE into the EvRT-DETR-B baseline effectively compresses the continuous event stream while introducing a negligible latency overhead of only 0.1 ms, which is 0.87--1.16\% of the inference latency.

However, applying PE alone leads to a slight performance drop, suggesting that standard backbones are not well-suited to fully exploit the resulting dense representation without dedicated temporal modeling. Incorporating the offline FAT module resolves this limitation, boosting performance to a peak mAP of 53.14. Notably, this gain is achieved without any additional parameters or inference latency compared to the PE-only variant.

Overall, FATE-E and FATE-S consistently outperforms the strong EvRT-DETR-B and S5-ViT-B baselines, respectively, with minimal overhead  of additional parameters and 0.1 ms ($\sim$1\%) end-to-end inference latency increase on an NVIDIA T4 GPU.

\section{Qualitative Results}
\label{app:qual}

Figure~\ref{fig:bbox_comparison} presents a qualitative comparison of bounding boxes at 20\,Hz. The strongest baselines exhibit missed detections and temporal instability, whereas FATE-E, equipped with our PE at inference, produces stable predictions, particularly for small or weakly activated targets.

Figures~\ref{fig:bbtracking_comparison} and~\ref{fig:bbtracking_comparison2} further highlight the impact of our FAT. The ablation shows that FATE-E trained with bounding box tracking achieves superior performance over the variant without detection-by-tracking supervision.

%\subsection{Qualitative Comparisons of Event %Detectors}

%\subsection{Comparisons under Different Frequencies}

\begin{figure}[h]
  \centering
  \includegraphics[width=.95\textwidth, keepaspectratio]{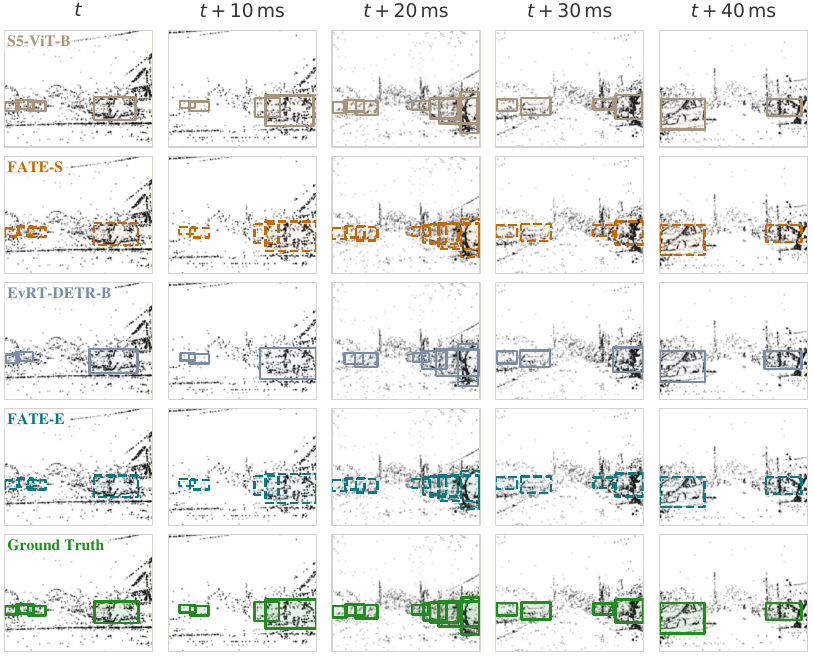}
  \caption{Qualitative comparison of  bounding boxes (Gen1, 20 Hz).}
  %The Gen1 test sequence 17-04-11\_15-13-23\_500000\_60500000 is used to produce the visualization. 
  \label{fig:bbox_comparison}
\end{figure}

\begin{figure}[h]
 \centering
  \includegraphics[width=\textwidth, keepaspectratio]{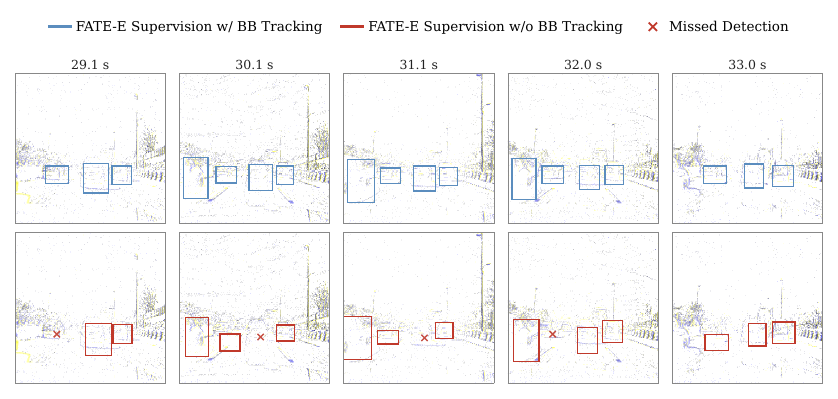}
  \caption{Visualization of the dense label supervision in FATE-E (Gen1, 20 Hz)}
  \label{fig:bbtracking_comparison}
%\end{figure}

\vspace{2cm}

%\begin{figure}[h]
%  \centering
  \includegraphics[width=\textwidth, keepaspectratio]{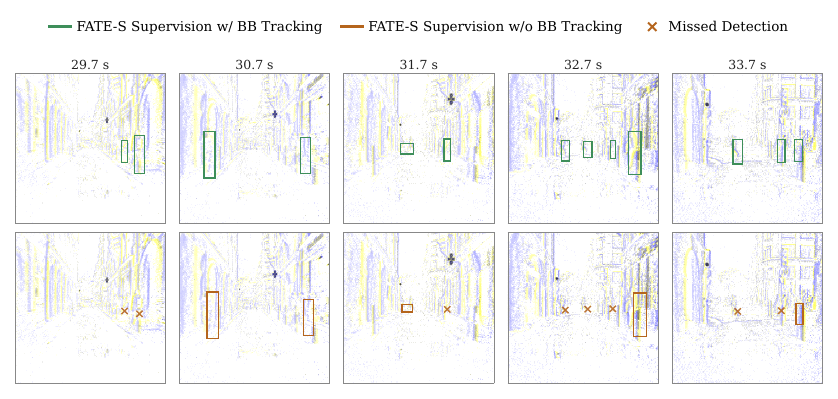}
  \caption{Visualization of the dense label supervision in FATE-S (Gen1, 20 Hz)}
  \label{fig:bbtracking_comparison2}
\end{figure}

\section{Reproduction Details}
\label{app:reproduce}
\input{sections/reproduce}

\section{Limitations}
\label{app:limit}
FATE is developed as a continuous-time encoding framework that integrates with standard synchronous detection backbones. Consequently, it operates on discrete accumulation windows $[t_1, t_2]$ to bridge asynchronous event streams and frame-based architectures. While the proposed time-normalized polynomial basis and FAT curriculum empirically yield strong robustness across a wide range of window durations and operating frequencies, eliminating the windowing mechanism entirely would require adopting fundamentally different downstream architectures (e.g., spiking neural networks or continuous-time state-space models). Such directions are orthogonal to our goal and are not necessary to validate the core contribution of this work, which is to improve temporal representations within existing detection pipelines.

Our formulation further adopts a fixed Legendre polynomial basis with trapezoidal quadrature. This design prioritizes simplicity, efficiency, and stability: low-order polynomials (e.g., $K=3$) provide a compact representation of transient, non-periodic event dynamics, while trapezoidal integration offers a robust approximation for irregular, bursty event streams. Although more expressive alternatives---such as learned implicit representations, adaptive basis functions, or spline-based encodings---could increase representational flexibility, they also introduce additional complexity and confounding factors.

In this work, we intentionally restrict the design space to isolate the effect of a lightweight, structured temporal encoding. Within this scope, FATE demonstrates strong empirical performance and efficiency across frequencies and resolutions. Exploring more adaptive or fully continuous-time formulations represents a promising direction for future work, but lies beyond the targeted scope of this paper.

\section{Broader Impacts}
\label{app:broader_impact}
The advancement of low-latency, event-based perception through the FATE framework has important implications for safety-critical systems. By improving robustness at high operating frequencies (e.g., 200 Hz), FATE can benefit applications that depend on fine-grained temporal resolution, including autonomous driving, aerial robotics, and mobile manipulation. In these settings, accurate high-frequency perception is critical for reliably capturing fast motion and reducing failure cases associated with motion blur or delayed reaction times. Furthermore, event-based vision possesses inherent privacy-preserving attributes; by recording only asynchronous brightness changes (e.g., structural edges and motion) rather than full-texture RGB frames, these sensors naturally obfuscate highly identifiable personal features compared to traditional cameras. In this sense, FATE effectively supports the stable and privacy-conscious deployment of event-based perception in real-world dynamic environments.

At the same time, these capabilities may also introduce dual-use considerations. The enhanced ability to reliably detect and track objects at extremely high frequencies and under challenging illumination conditions could be applied in domains such as pervasive surveillance or military sensing systems, where automated tracking and decision-making raise ethical concerns.

FATE is a general-purpose perception module that operates purely at the representation and detection level, without incorporating application-specific decision-making or control logic. As such, it does not in itself require task- or domain-specific safeguards. Its real-world impact is instead determined by the downstream systems in which it is integrated, where appropriate safety measures, governance, and regulatory oversight should be applied. We encourage continued discussion within the community on the responsible deployment of high-frequency perception systems.

% The advancement of low-latency, event-based perception through the FATE framework carries important societal implications. On the positive side, enabling robust object detection at high frequencies (e.g., 200 Hz) directly benefits safety-critical autonomous systems. Applications such as autonomous driving, unmanned aerial vehicles, and robotics rely on fine-grained temporal information to prevent collisions. By mitigating high-frequency performance degradation in event cameras, FATE brings high-dynamic-range, motion-blur-free perception closer to reliable real-world deployment.

% However, these capabilities also introduce potential ethical concerns and dual-use risks. In particular, the ability to reliably track high-speed objects under challenging lighting conditions could be leveraged for pervasive surveillance or military use, such as automated targeting or drone tracking.

% Moving forward, it is important for the research community to advocate for responsible regulatory frameworks governing high-speed automated surveillance and to promote safeguards that mitigate misuse.

% --- SECTION E ---
\section{Public Resources Used}
\label{app:resources}
In this section, we acknowledge the public resources used during the course of this work.

\subsection{Public Datasets Used}
\begin{itemize}
    \item \textbf{Gen1} \dotfill Prophesee Gen1 Automotive Detection Dataset License
    \item \textbf{1 Mpx} \dotfill Prophesee 1Mpx Automotive Detection Dataset License
\end{itemize}
\subsection{Public Implementations Used}
\begin{itemize}
    \item \textbf{RVT} \dotfill MIT License
    \item \textbf{SSM} \dotfill Apache-2.0 license
    \item \textbf{LEOD} \dotfill MIT License
    \item \textbf{EvRT-DETR} \dotfill BSD-3 License
\end{itemize}
% \subsection{Additional Code and License Attributions}
% This repository includes code from the following projects:

% \begin{itemize}
%     \item \textbf{LeanBase} (\texttt{evlearn/bundled/leanbase/}) \\
%     License: BSD-2 \\
%     Purpose: Primitive PyTorch routines.

%     \item \textbf{RT-DETR (PyTorch)} (\texttt{evlearn/bundled/rtdetr\_pytorch/}) \\
%     Original Project: \url{https://github.com/lyuwenyu/RT-DETR} \\
%     License: Apache-2.0 \\
%     Purpose: Reference RT-DETR implementation.

%     \item \textbf{YOLOX} (\texttt{evlearn/bundled/yolox/}) \\
%     Original Project: \url{https://github.com/Megvii-BaseDetection/YOLOX} \\
%     License: Apache-2.0 \\
%     Purpose: Reference YOLOX implementation.
% \end{itemize}

%% file: sections/justify3.tex
\newtheorem{prop}{Proposition}

In this section, we provide the theoretical foundation for our Pillar Encoding (PE) formulation. To compress the continuous event signal, we project it onto an orthogonal polynomial basis using Legendre polynomials, $L_k(\tau)$, which satisfy the standard orthogonality and normalization condition $\langle L_k, L_m \rangle = \frac{2}{2k+1} \delta_{km}$ over the interval $[-1, 1]$ \cite{abramowitz1964handbook}. Building on this formulation, we establish two key guarantees of our architecture: (i) the truncated Legendre projection yields the unique minimizer of the $L^2$ functional reconstruction error over the polynomial subspace; and (ii) the density-aware numerical quadrature yields an unbiased temporal representation.

\subsection{Optimality of the Truncated Orthogonal Projection}

We model the temporal trajectory of the event features within a pillar as an underlying continuous, square-integrable function $h(\tau) \in \mathcal{H}$, where $\mathcal{H} = L^2([-1,1])$ is a Hilbert space with the standard inner product $\langle f, g \rangle = \int_{-1}^{1} f(\tau)g(\tau)d\tau$. 

\begin{prop}[Optimality of Legendre Projection]
Let $h \in L^2([-1,1])$, and let $\mathcal{S}_K = \mathrm{span}\{L_0, L_1, \dots, L_{K-1}\}$, where $\{L_k\}_{k=0}^{K-1}$ are the Legendre polynomials, orthogonal with respect to the $L^2([-1,1])$ inner product. Define
$$
h_K(\tau) = \sum_{k=0}^{K-1} a_k L_k(\tau), \quad \text{with} \quad a_k = \frac{2k+1}{2} \langle h, L_k \rangle.
$$
Then $h_K$ is the unique orthogonal projection of $h$ onto $\mathcal{S}_K$, and hence $h_K = \arg\min_{g \in \mathcal{S}_K} \| h - g \|_{L^2}^2$.
\end{prop}

\begin{proof}
Let $g(\tau) \in \mathcal{S}_K$ be an arbitrary polynomial in the subspace. We can analyze the squared $L^2$ error between the true signal $h$ and the approximation $g$:
$$
\| h - g \|^2 = \| (h - h_K) + (h_K - g) \|^2.
$$
Expanding this expression using the properties of the inner product yields:
$$
\| h - g \|^2 = \| h - h_K \|^2 + \| h_K - g \|^2 + 2\langle h - h_K, h_K - g \rangle.
$$
To show that the cross-term vanishes, we must verify that the residual $(h - h_K)$ is orthogonal to every basis function $L_m$ for $m \in \{0, \dots, K-1\}$. Evaluating the inner product:
$$
\langle h - h_K, L_m \rangle = \langle h, L_m \rangle - \sum_{k=0}^{K-1} a_k \langle L_k, L_m \rangle.
$$
By the orthogonality of the Legendre polynomials, $\langle L_k, L_m \rangle = \frac{2}{2m+1}\delta_{km}$, so the sum collapses to a single term:
$$
\langle h - h_K, L_m \rangle = \langle h, L_m \rangle - a_m \left(\frac{2}{2m+1}\right).
$$
Substituting the definition of $a_m = \frac{2m+1}{2} \langle h, L_m \rangle$, we obtain:
$$
\langle h - h_K, L_m \rangle = \langle h, L_m \rangle - \langle h, L_m \rangle = 0.
$$
Since $(h - h_K)$ is orthogonal to every basis vector in $\mathcal{S}_K$, and $(h_K - g)$ is a linear combination of these basis vectors, it follows that $\langle h - h_K, h_K - g \rangle = 0$. We are left with:
$$
\| h - g \|^2 = \| h - h_K \|^2 + \| h_K - g \|^2.
$$
Since $\| h - h_K \|^2$ is fixed for a given $h$ and $K$, minimizing the total error requires minimizing $\| h_K - g \|^2$. By positive-definiteness of the norm, the minimum is achieved if and only if $\| h_K - g \|^2 = 0$, which implies $g = h_K$. 

This proves that PE optimally projects the continuous feature trajectory into the $K$-dimensional polynomial coefficient space without relying on fixed temporal binning. 
\end{proof}

\subsection{Unbiasedness of the Quadrature-Corrected Estimator}

Event streams are asynchronous, meaning the temporal distribution of events is non-uniform. If an object accelerates, or if contrast changes, events arrive in dense bursts.

\begin{prop}[Bias of Naive Aggregation vs. Quadrature]
Let $\{\tau_n\}_{n=1}^N$ be an ordered sequence of timestamps with $\tau_1 < \cdots < \tau_N$, sampled according to an inhomogeneous temporal density $\rho(\tau)$. Define $H_n := h(\tau_n)$, where $h : [-1,1] \to \mathbb{R}$ is the underlying continuous-time signal. Let $\Delta \tau_n := \tfrac{1}{2}(\tau_{n+1} - \tau_{n-1})$ denote the local temporal spacing for interior points, with appropriate one-sided definitions at the boundaries.

The naive empirical mean estimator $\hat{I}_k^{\mathrm{naive}}$ is biased by $\rho(\tau)$, whereas the quadrature estimator $\hat{I}_k^{\mathrm{quad}}$ asymptotically recovers the unscaled projection integral $I_k$.
\end{prop}

\begin{proof}
The unscaled $k$-th Legendre projection integral is defined as:
$$
I_k = \int_{-1}^{1} h(\tau) L_k(\tau)\, d\tau.
$$

The naive Monte Carlo estimator aggregates features uniformly across events:
$$
\hat{I}_k^{\mathrm{naive}} = \frac{2}{N} \sum_{n=1}^N H_n L_k(\tau_n).
$$

As $N \to \infty$, by the Law of Large Numbers, this estimator converges to the expectation under $\rho(\tau)$:
$$
\lim_{N \to \infty} \hat{I}_k^{\mathrm{naive}} \propto \int_{-1}^{1} h(\tau) L_k(\tau)\, \rho(\tau)\, d\tau \neq I_k.
$$

Thus, $\hat{I}_k^{\mathrm{naive}}$ biases the representation toward regions with high event density ($\rho(\tau)$), distorting the true temporal trajectory $h(\tau)$.

Conversely, the trapezoidal quadrature estimator uses temporal weights $w_n \approx \Delta \tau_n$ \cite{press2007numerical, davis2007methods}:
$$
\hat{I}_k^{\mathrm{quad}} = \sum_{n=1}^N w_n H_n L_k(\tau_n).
$$

This corresponds to a Riemann sum over a non-uniform partition. As the maximum spacing vanishes ($\max_n \Delta \tau_n \to 0$), the sum converges to the integral \cite{rudin1976principles}:
$$
\lim_{\max_n \Delta \tau_n \to 0} \sum_{n=1}^N w_n H_n L_k(\tau_n)
= \int_{-1}^{1} h(\tau) L_k(\tau)\, d\tau = I_k.
$$

Therefore, the quadrature estimator yields an asymptotically unbiased estimate of the unscaled continuous-time projection integral. 

In our final PE formulation (Section \ref{sec:pillar_encoding}), the network computes the feature $z_{c,j,k} = \frac{1}{W_j} \hat{I}_k^{\mathrm{quad}}$. The pillar-specific normalization term $W_j$ converts the raw integral into a duration-normalized average. Rather than claiming strict $L^2$-optimality at this final output stage, we note that this formulation ensures the model learns a density-invariant, duration-normalized temporal feature, stabilizing the inputs for subsequent network layers.
\end{proof}

%% file: tables/7.tex
\begin{table}[t]
\centering
\caption{Quality of generated pseudo labels on Gen1 for cars and pedestrians. We sample 10k boxes per frequency and evaluate against ground truth nearest in time.}
\label{tab:pseudo_quality_class}
\setlength{\tabcolsep}{4pt}
\scriptsize
\begin{tabular}{llcccccccc}
\toprule
\multirow{2}{*}{Model} & \multirow{2}{*}{Freq.}
& \multicolumn{4}{c}{Cars}
& \multicolumn{4}{c}{Pedestrians} \\
\cmidrule(lr){3-6} \cmidrule(lr){7-10}
& 
& Precision$\uparrow$ & Recall$\uparrow$ & Med IoU$\uparrow$ & IoU $>0.5$ $\uparrow$
& Precision$\uparrow$ & Recall$\uparrow$ & Med IoU$\uparrow$ & IoU $>0.5$ $\uparrow$ \\
\midrule
\multirow{5}{*}{PE + S5-ViT-B}
& 20 Hz  & 0.78 & \underline{0.64} & \underline{0.62} & 0.78 & 0.68 & \underline{0.52} & \underline{0.54} & 0.68 \\
& 40 Hz  & 0.75 & 0.60 & 0.59 & 0.74 & 0.65 & 0.48 & 0.51 & 0.64 \\
& 80 Hz  & 0.70 & 0.54 & 0.55 & 0.69 & 0.60 & 0.43 & 0.47 & 0.60 \\
& 100 Hz & 0.66 & 0.50 & 0.52 & 0.66 & 0.57 & 0.40 & 0.44 & 0.56 \\
& 200 Hz & 0.60 & 0.42 & 0.47 & 0.45 & 0.49 & 0.32 & 0.40 & 0.49 \\
\midrule
\multirow{5}{*}{PE + EvRT-DETR-B}
& 20 Hz  & \textbf{0.83} & \textbf{0.68} & \textbf{0.65} & \textbf{0.83} & \textbf{0.72} & \textbf{0.56} & \textbf{0.57} & \textbf{0.72} \\
& 40 Hz  & \underline{0.80} & \underline{0.64} & \underline{0.62} & \underline{0.80} & \underline{0.69} & \underline{0.52} & \underline{0.54} & \underline{0.69} \\
& 80 Hz  & 0.75 & 0.58 & 0.58 & 0.75 & 0.64 & 0.47 & 0.50 & 0.48 \\
& 100 Hz & 0.72 & 0.54 & 0.55 & 0.72 & 0.61 & 0.44 & 0.47 & 0.44 \\
& 200 Hz & 0.65 & 0.46 & 0.48 & 0.47 & 0.53 & 0.36 & 0.42 & 0.38 \\
\bottomrule
\end{tabular}
\end{table}

\begin{table}[h]
\centering
\caption{FAT vs. ablated strategies in FATE-E on Gen1.}
\label{tab:fat_component_ablation}
\setlength{\tabcolsep}{4pt}
\scriptsize
\begin{tabular}{lcccccccc}
\toprule
\multirow{2}{*}{Method} 
& \multirow{2}{*}{Ground Truth}
& \multirow{2}{*}{Dense Sup.} 
& \multirow{2}{*}{Pseudo-labels} 
& \multicolumn{5}{c}{mAP (\%) $\uparrow$} \\ 
\cmidrule(lr){5-9}
&  &  &  
& 20Hz & 40Hz & 80Hz & 100Hz & 200Hz \\
\midrule
PE + EvRT-DETR-B
& \cmark & \xmark & \xmark
& 51.94 & 50.31 & 46.06 & 44.10 & 36.40  \\

Multi-frequency  
& \cmark & \xmark & \xmark
& 49.24 & 47.06 & 45.70 & 43.12 & 40.48 \\

Consistency-only 
& \cmark & \xmark & \cmark
& 48.90 & 47.85 & 45.40 & 43.80 & 41.10  \\

Naive interpolation 
& \cmark & \cmark & \xmark 
& \underline{51.18} & \underline{49.37} & \underline{47.21} & \underline{45.41} & \underline{41.64} \\

FAT (ours) 
& \cmark & \cmark & \cmark
& \textbf{53.14} & \textbf{51.60} & \textbf{48.40} & \textbf{46.80} & \textbf{42.38} \\
\bottomrule
\end{tabular}
\end{table}

%% file: tables/4.tex
\paragraph{Impact of temporal polynomial degree and pooling mechanisms.} We conduct a comprehensive sensitivity analysis of the temporal polynomial degree ($K$) and pooling strategies on the Gen1 benchmark. 

In Table~\ref{tab:epe_temporal_sensitivity}, performance improves as $K$ increases from 2 to 3, but degrades for $K = 4$ and $5$ despite the added complexity. Accordingly, we set $K=3$, which achieves the best trade-off between performance and complexity. Moreover, we adopt a temporal basis with trapezoidal weighting, which consistently outperforms the baseline strategies.

Since these hyperparameters primarily control temporal feature compression rather than spatial resolution, the optimal configuration identified in Table~\ref{tab:epe_temporal_sensitivity} ($K=3$ with trapezoidal weighting) generalizes effectively to higher-resolution settings, as reflected in the 1Mpx results reported in the main text.

\begin{table*}[h]
\centering
\caption{Sensitivity of temporal polynomial degree $K$ and temporal pooling methods in PE across operating frequencies on Gen1 test set.}
\label{tab:epe_temporal_sensitivity}
\setlength{\tabcolsep}{4pt}
\scriptsize
\begin{tabular}{lccc ccc ccc}
\toprule
\multirow{2}{*}{Variant}
& \multicolumn{3}{c}{20\,Hz}
& \multicolumn{3}{c}{100\,Hz}
& \multicolumn{3}{c}{200\,Hz} \\
\cmidrule(lr){2-4} \cmidrule(lr){5-7} \cmidrule(lr){8-10}
& mAP & AP$_{50}$ & AP$_{75}$
& mAP & AP$_{50}$ & AP$_{75}$
& mAP & AP$_{50}$ & AP$_{75}$ \\
\midrule
$K=2$ & 46.32 & 76.00 & 49.80 & 38.20 & 65.50 & 40.30 & 29.90 & 54.80 & 30.20 \\
$K=3$ & \textbf{51.94} & \textbf{81.20} & \textbf{55.10} & \textbf{44.10} & \textbf{71.90} & \textbf{46.20} & \textbf{36.40} & \textbf{61.80} & \textbf{36.80} \\
$K=4$ & \underline{50.71} & \underline{79.95} & \underline{53.90} & 42.82 & 70.40 & \underline{44.80} & \underline{34.88} & \underline{59.90} & \underline{35.10}  \\
$K=5$ & 49.80 & 78.90 & 52.70 & \underline{43.50} & \underline{70.50} & 44.22 & 33.60 & 58.20 & 33.80 \\ 
\midrule
Max Pooling & 46.10 & 72.80 & 50.50 & 38.90 & 66.10 & 42.00 & 29.40 & 54.20 & 29.80 \\
Temporal Basis + Uniform Weights & \underline{51.63} & \underline{80.90} & \underline{54.80} & \underline{43.70} & \underline{71.20} & \underline{45.60} & \underline{35.60} & \underline{60.50} & \underline{35.70} \\
Temporal Basis + Trapezoidal Weights & \textbf{51.94} & \textbf{81.20} & \textbf{55.10} & \textbf{44.10} & \textbf{71.90} & \textbf{46.20} & \textbf{36.40} & \textbf{61.80} & \textbf{36.80} \\
\bottomrule
\end{tabular}
\end{table*}

% ---------- Previous Tables ----------

\begin{comment}
\begin{table}[t]
\centering
\caption{Ablation study on \color{red}{time to detection} (Gen1 test).}
\label{tab:ablation_latency}
\setlength{\tabcolsep}{3.5pt}
\scriptsize
\begin{tabular}{l ccc cc}
\toprule
\multirow{2}{*}{Setting} 
& \multicolumn{3}{c}{\textbf{TTD (ms)} $\downarrow$} 
& \multicolumn{2}{c}{\textbf{Reliability}} \\
\cmidrule(lr){2-4} \cmidrule(lr){5-6}
& Median & P90 & Mean 
& Miss@200ms (\%) $\downarrow$ 
& $P(\le100$ms$)$ (\%) $\uparrow$ \\
\midrule
RVT-B & -- & -- & -- & -- & -- \\
EPE (+RVT-B) & -- & -- & -- & -- & -- \\
EPE (+RVT-B) + FALR & -- & -- & -- & -- & -- \\
\bottomrule
\end{tabular}
\end{table}  
\end{comment}

% \begin{table}[t]
% \centering
% \caption{Sensitivity of the temporal pooling choices in EPE.}
% \label{tab:epe_temporal_sensitivity}
% \setlength{\tabcolsep}{6pt}
% \scriptsize
% \begin{tabular}{lcccc}
% \toprule
% Variant & mAP & AP$_{50}$ & AP$_{75}$ & Params (M) \\
% \midrule
% $K=2$ & -- & -- & -- & -- \\
% $K=3$ & -- & -- & -- & -- \\
% $K=4$ & -- & -- & -- & -- \\
% %$K=5$ & -- & -- & -- & -- \\
% \midrule
% %Mean pooling & -- & -- & -- & -- \\
% Max Pooling & -- & -- & -- & -- \\
% Temporal Basis + Uniform Weights & -- & -- & -- & -- \\
% Temporal Basis + Trapezoidal Weights & -- & -- & -- & -- \\
% \bottomrule
% \end{tabular}
% \end{table}

%% file: tables/5.tex
\paragraph{Impact of the pillar size.}
We perform a spatial pillar size ablation at an intermediate operating frequency of 100 Hz. Since spatial quantization primarily affects the morphological resolution of targets rather than their temporal dynamics, the optimal spatial scale remains largely invariant to the choice of temporal frequency.

Table~\ref{tab:ablation_epe} highlights a key architectural trade-off between spatial resolution and global context preservation. As the pillar size decreases from 8 to 2, performance improves steadily, reaching a peak mAP of 46.80\%, indicating that finer spatial discretization more effectively captures the fine-grained structure of moving targets. 

However, further reducing the pillar size to 1 leads to a performance drop. This inflection arises from the fixed resolution budget of $P=16,000$ pillars per sample; at a pillar size of 1, the grid expands to 81,920 possible pillars, requiring the model to discard over 80\% of spatial locations. In practice, a pillar size of 2 provides the best trade-off---preserving spatial fidelity while avoiding severe information loss due to pillar truncation.

%\hl{For Gen1, event statistics indicate that $N = 32$ is sufficient for the majority of the active pillars, particularly at high temporal resolution (5-10) ms, whereas 99th percentile event count using pillar size of 2 remains below 32. The number of events exceeds N, in small tail ( around 1\%) of active pillars at larger event accumulation window (e.g. 50ms). Therefore, $ N = 32$ captures the majority of local temporal activity, and avoids allocating large buffers for inherently sparse pillars.}

On Gen1, event statistics show that $N=32$ (samples per pillar) is sufficient for over 99\% of active pillars (using a spatial size of $2 \times 2$). Only a small $~\sim$1\% tail exceeds this threshold, primarily during longer 50 ms accumulation windows. At higher temporal resolutions (5--10 ms), event counts fall well below 32. Thus, setting $N=32$ safely captures local temporal dynamics while avoiding excessive memory waste in sparse pillars.

\begin{table}[h]
\centering
\caption{Sensitivity of pillar size choices in PE on the Gen1 test set at 100\,Hz.}x
\label{tab:ablation_epe}
\setlength{\tabcolsep}{4pt}
\small
{Fixed: $P=16{,}000,\; C=64,\; K=3,\; N=32$}\\[3pt]
\begin{tabular}{@{}ccccccc@{}}
\toprule
Pillar size & Grid $(H\times W)$ & Total pillars & Max pillars $P$ & mAP (\%) $\uparrow$ & AP$_{50}$ (\%) $\uparrow$ & AP$_{75}$ (\%) $\uparrow$ \\
\midrule
8 & $32\times40$   & 1,280  & 1,280  & 43.18 & 69.20 & 45.10 \\
4 & $64\times80$   & 5,120  & 5,120  & 45.73 & \underline{72.60} & 47.80 \\
\textbf{2} & $\mathbf{128\times160}$ & \textbf{20,480} & \textbf{16,000} & \textbf{46.80} & \textbf{73.10} & \textbf{49.10} \\
1 & $256\times320$ & 81,920 & 16,000 & \underline{46.09} & 71.90 & \underline{48.20} \\
\bottomrule
\end{tabular}
\end{table}

%% file: tables/8.tex
\begin{table*}[h]
\centering
\caption{Impact of temporal capacity $C$ on the Gen1 test set.}
\label{tab:fe_channel_ablation}
\setlength{\tabcolsep}{4pt}
\small
{Fixed: Pillar size $=$ 2, $P=16,000, K=3, N=32$}\\[3pt]
\begin{tabular}{lccc ccc ccc c}
\toprule
\multirow{2}{*}{Channel dimension $C$}
& \multicolumn{3}{c}{20\,Hz}
& \multicolumn{3}{c}{100\,Hz}
& \multicolumn{3}{c}{200\,Hz}
& \multirow{2}{*}{\#Params (M)} \\
\cmidrule(lr){2-4} \cmidrule(lr){5-7} \cmidrule(lr){8-10}
& mAP & AP$_{50}$ & AP$_{75}$
& mAP & AP$_{50}$ & AP$_{75}$
& mAP & AP$_{50}$ & AP$_{75}$ \\
\midrule
$C=32$
& 52.48 & 81.70 & 56.10
& 45.71 & 71.90 & 47.70
& 40.64 & 65.90 & 40.80
& 57.151 \\

$C=48$
& 52.86 & 82.00 & 56.50
& 46.22 & 72.50 & 48.20
& 41.72 & 67.10 & 41.70
& 57.156 \\

$C=64$
& \textbf{53.14} & \textbf{82.30} & \textbf{56.80}
& \textbf{46.80} & \textbf{73.10} & \textbf{49.10}
& \textbf{42.38} & \textbf{68.20} & \textbf{42.50}
& 57.161 \\

$C=80$
& \underline{53.08} & \underline{82.20} & \underline{56.70}
& \underline{46.69} & \underline{72.90} & \underline{48.70}
& \underline{42.21} & \underline{67.90} & \underline{42.20}
& 57.166 \\

$C=96$
& 52.97 & 82.00 & 56.40
& 46.41 & 72.60 & 48.30
& 41.95 & 67.40 & 41.90
& 57.171 \\
\bottomrule
\end{tabular}
\end{table*}

\begin{table}[h]
\centering
\caption{Efficiency ablations of FATE-E and FATE-S on the Gen1 test set at 20, 100, and 200\,Hz.}
\label{tab:efficiency}
\setlength{\tabcolsep}{4pt}
\scriptsize
\begin{tabular}{l c cc cc cc}
\toprule
\multirow{2}{*}{Model} & \multirow{2}{*}{Params (M)$\downarrow$}
& \multicolumn{2}{c}{20\,Hz} 
& \multicolumn{2}{c}{100\,Hz}
& \multicolumn{2}{c}{200\,Hz} \\
\cmidrule(lr){3-4} \cmidrule(lr){5-6} \cmidrule(lr){7-8}
& & mAP$\uparrow$ & Latency (ms)$\downarrow$
  & mAP$\uparrow$ & Latency (ms)$\downarrow$
  & mAP$\uparrow$ & Latency (ms)$\downarrow$ \\
\midrule
EvRT-DETR-B      & \textbf{57.14} & 52.70 & \textbf{11.4} & 42.20 & \textbf{11.4} &  32.00 & \textbf{11.4} \\
+ PE             & 57.16 & 51.94 & 11.5 &  44.10   & 11.5 &  36.40  & 11.5 \\
+ PE + FAT (FATE-E)      & 57.16 & \textbf{53.14} & 11.5 & \textbf{46.80} & 11.5 & \textbf{42.38} & 11.5 \\
\midrule
S5-ViT-B         & \textbf{18.19} & 47.40 & \textbf{8.5} & 42.49 & \textbf{8.5} & 39.84 & \textbf{8.5} \\
+ PE             & 18.33 & 47.70 & 8.6 & 43.08 & 8.6 & 39.30 & 8.6 \\
+ PE + FAT (FATE-S)      & 18.33 & \textbf{48.20} & 8.6 & \textbf{43.18} & 8.6 & \textbf{40.80} & 8.6 \\
\bottomrule
\end{tabular}
\end{table}

%% file: sections/reproduce.tex
This section provides implementation details, hyperparameter settings, and dataset preprocessing procedures that were omitted from the main text for brevity. An anonymized ZIP archive is included with the submission, and the code along with the high-frequency event-based datasets will be released publicly upon acceptance.

%\subsection{Algorithmic Implementations}
%\label{sec:algorithms}
%\input{sections/algorithms}

\subsection{Datasets and Preprocessing}
We evaluate our method and the baselines on two standard event-based object detection benchmarks, Gen1~\cite{tournemireLargeScaleEventbased2020} and 1Mpx~\cite{perotLearningDetectObjects}, which differ substantially in resolution and annotation density. In particular, Gen1 provides sparse annotations (1–4 Hz), requiring stronger temporal aggregation or densification strategies, whereas 1Mpx offers high frame-rate, dense annotations ($\sim$60 Hz), enabling finer-grained temporal supervision.

To ensure consistency with prior work and the official EvRT-DETR~\cite{torbunovEvRTDETRLatentSpace} and S5-ViT-B \cite{zubicStateSpaceModels2024} implementations, we preserve the original input resolution settings after preprocessing. Specifically, Gen1 samples are padded to $256\times320$, while 1Mpx samples are resized to $360\times640$ and subsequently padded to $384\times640$.

To assess robustness across different temporal resolutions, we vary the event accumulation window to simulate multiple sensing frequencies, including 50 ms (20 Hz), 25 ms (40 Hz), 12.5 ms (80 Hz), 10 ms (100 Hz), and 5 ms (200 Hz). Higher temporal frequencies correspond to shorter accumulation windows, resulting in sparser event representations and reduced temporal context per input sample.

\subsection{Pillar Encoding Configuration}

We tune the architectural hyperparameters for Pillar Encoding---specifically the latent channel dimension $C$, the Legendre basis size $K$, and the spatial pillar size---through sensitivity ablations on the Gen1 dataset. For pillar size 2, we use $D = 7$ discarding the center offsets of events. Based on the results detailed in Tables~\ref{tab:epe_temporal_sensitivity}, \ref{tab:ablation_epe}, and \ref{tab:fe_channel_ablation} (Appendix~\ref{appendix:result_details}), we adopt $C=64$, $K=3$, and a pillar size of 2 for both FATE-E and FATE-S, as this configuration consistently achieves the strongest detection performance across all evaluated operating frequencies.

%yields the best detection performance across all evaluated operating frequencies.

\subsection{Data Augmentation}
To improve the performance and generalization of the proposed model while mitigating overfitting, we employ a comprehensive set of data augmentation techniques. These transformations encourage the model to learn features that are robust to variations in orientation, scale, and partial occlusion. The augmentation pipeline used during training is as follows:

\begin{itemize}
\item \textit{Random Horizontal Flip:} Applied with probability $0.5$ to account for mirror symmetry in natural scenes.

\item \textit{Random Rotation:} Input event streams are rotated by an angle $\theta \in [-30^\circ, 30^\circ]$, improving robustness to non-upright sensor orientations.

\item \textit{Random Affine Translation:} Spatial shifts of up to $0.5$ (relative to image dimensions) are applied along both $x$ and $y$ axes to increase robustness to object displacement and centering variability.

\item \textit{Scale Jitter:} Inputs are randomly scaled by a factor $s \in [0.5, 1.5]$, simulating variations in object distance and promoting multi-scale feature learning.

\item \textit{Shear:} A shearing transformation up to $30^\circ$ is applied to model perspective distortions and oblique viewpoints.

\item \textit{Center Crop:} Used to maintain consistent input dimensions while preserving central spatial structure.

\item \textit{Random Erasing:} Applied with probability $0.4$, where a randomly selected rectangular region is replaced with noise or zero values, simulating occlusions and encouraging reliance on partial and local features.
\end{itemize}

\subsection{Multi-Frequency Pseudo-Label Generation}
To densify supervision during Phase 1 of our Frequency-Aware Training, we apply the event detector (augmented by PE) sequentially over the full training split using the canonical event window of 50 ms. Recurrent states are preserved across clips of 21 frames for Gen1 and 10 frames for 1Mpx to maintain temporal continuity. Ground-truth bounding boxes are retained strictly at annotated timestamps, while pseudo-labels are generated only for unlabeled timestamps corresponding to the target temporal frequencies.

To ensure high-quality pseudo-supervision, raw detections are first filtered prior to temporal association. We apply an initial IoU threshold of 0.3 for tacking, followed by class-specific detection confidence thresholds of 0.6 for cars and 0.3 for pedestrians, applied to both objectness and classification scores. Subsequently, standard dataset-specific bounding box filtering is performed.

%\textcolor{red}{Which detector was used to generate bounding boxes?}

Finally, a tracking-by-detection post-processing step is used to enforce temporal consistency. To suppress short, noisy, or spurious trajectories produced by the event detector, we impose a minimum tracklet length $L_{\min}$. To account for increasing temporal resolution at higher sampling frequencies, this constraint is scaled proportionally, with $L_{\min}=6$ at 20 Hz, $L_{\min}=12$ at 40 Hz, and further increased accordingly for 80 Hz, 100 Hz, and 200 Hz.

\subsection{Linear Curriculum Schedule}
We employ a curriculum over the frequency-specific datasets. At training round $r$, the student frequency is sampled according to $f \sim p_r(f)$, which gradually shifts probability mass toward higher, sparser frequencies.

To ensure stable training from dense to sparse regimes, we adopt a linear schedule over the sampling distribution. Let $\mathcal{F} = \{f_{\min}, \dots, f_{\max}\}$ denote the ordered set of frequencies (e.g., $20, 40, 80, 100, 200$\,Hz), and let $r \in [0, R]$ denote the current training epoch. We define the normalized progress $\alpha_r = r / R \in [0,1]$ to bias sampling toward higher frequencies as $r$ increases.

The sampling distribution is given by:
$$
p_r(f_i) = \frac{(1 - \alpha_r)\,\mathbf{1}_{f_i = f_{\min}} + \alpha_r \, w_i}{Z_r},
$$
where $w_i = \frac{i}{|\mathcal{F}|}$ and $Z_r$ is a normalization constant ensuring $\sum_{f_i \in \mathcal{F}} p_r(f_i)=1$.

This schedule starts with sampling concentrated at the canonical frequency $f_{\min}$ and gradually shifts toward higher frequencies as $\alpha_r$ increases. Frequencies are sampled independently per training sample within a mini-batch, while the teacher network consistently operates at the canonical frequency $f_c$. We use this fixed linear curriculum for all experiments without additional hyperparameter tuning.

\subsection{Tuning Hyperparameters}
Models are trained using the AdamW optimizer with a weight decay of $1 \times 10^{-4}$ and gradient clipping with a maximum norm of $5.0$.

For the Gen1 variants, models are trained for $400$ epochs with a batch size of $32$. The temporal detectors are optimized using a OneCycle learning rate schedule \cite{smith2019super} with a peak learning rate of $2\times10^{-4}$ and a warm-up proportion of $0.005$. Concurrently, the teacher network is updated via an exponential moving average (EMA) using a momentum decay of $\gamma = 0.9999$.
For the 1Mpx variants, we maintain identical optimizer and scheduler configurations. Across all experimental settings, the transformer head utilizes 300 object queries and 100 denoising queries to facilitate robust bipartite matching and bounding box regression.

%For the 1Mpx variants, we use the same optimizer and scheduler configurations. Across all settings, the transformer head employs $300$ object queries and $100$ denoising queries to support robust bipartite matching and bounding box regression.

The event-wise MLP in Section~\ref{sec:pillar_encoding} is trained jointly with the detector in an end-to-end manner. Specifically, each $D$-dimensional augmented event vector is projected via a shared linear layer followed by BatchNorm and ReLU into a $C$-dimensional latent embedding, and all MLP parameters are optimized together with the backbone and detection head using the same training configuration as the base detector. In the temporal encoding stage, the coefficients $\alpha_{c,k}$ and $\beta_c$ (Eq.~\ref{eq:represent}) are not manually tuned hyperparameters; instead, they are learned end-to-end as trainable parameters that linearly combine the $K$ quadrature-corrected temporal moments within each channel.

% \textcolor{red}{How did you tune $\lambda_{\mathrm{iou}}$ (Eq.~\ref{eq:box-loss}) and $p_v$ (Eq.~\ref{eq:w_v})?}

We set the box regression weights of $\lambda_{\ell_1}$ to 5 and $\lambda_{iou}$ to 2 (Eq.~\ref{eq:box-loss}), commonly used for the bounding-box and GIoU terms.
For the S5-ViT-B variant, we retain the YOLOX detection objective, where the box term is an IoU regression loss weighted by 5, together with objectness and classification losses. These values are standard for each category of detectors and needed no hyperparameter tuning. 

% For Phase 2 training, pseudo-label supervision is weighted per instance by $p_v$, the teacher confidence score, with ground-truth assigned weight 1.0.

All experiments are conducted on two NVIDIA A40 GPUs (48 GB GDDR6 each) for training and an NVIDIA T4 GPU (16 GB GDDR6) for evaluation. Training requires approximately two weeks per run, while evaluation takes about one week per method–dataset pair.

%\textcolor{red}{/* Describe how you train the MLP (AdamW?) and tune the hyperparameters such as $\alpha$, $\beta$, $C$, and $K$. Also, describe what values of $\alpha$, $\beta$, $C$, and $K$ are used. Finally, briefly say "We used NVIDIA GPUs XYZ for training and evaluation. */}